\documentclass{article}

\PassOptionsToPackage{sort,numbers}{natbib}

\usepackage[final]{neurips_2022}

\usepackage[utf8]{inputenc} 
\usepackage[T1]{fontenc}    
\usepackage[colorlinks=true,citecolor=blue,pagebackref=true,breaklinks=true,bookmarks=false,linkcolor=blue]{hyperref} 
\usepackage{url}            
\usepackage{booktabs}       
\usepackage{amsfonts}       
\usepackage{nicefrac}       
\usepackage{microtype}      
\usepackage{xcolor}         
\usepackage{enumitem}
\usepackage{subfigure}
\usepackage{graphicx}
\usepackage{amsmath}
\usepackage{amssymb,amsthm}

\usepackage{times}
\usepackage{epsfig}
\usepackage{graphicx}
\usepackage{float}
\usepackage{wrapfig}
\usepackage{amsmath,amssymb,amsthm}
\usepackage{algorithm,algorithmicx,algpseudocode}
\usepackage{bm,xspace}
\usepackage{comment}
\usepackage{verbatim}
\usepackage{multirow}
\usepackage{balance}
\usepackage{url}
\usepackage{booktabs}
\usepackage{etoolbox,siunitx}
\usepackage{calc}
\usepackage{pifont,hologo}
\usepackage{nicefrac}
\usepackage{dsfont}
\usepackage{xcolor}
\usepackage{blindtext}
\usepackage{tabularx}

\usepackage[super]{nth}
\usepackage{nicefrac}
\robustify\bfseries

\def\dd{\mathbf{d}}
\def\ee{\mathbf{e}}

\def\ww{\mathbf{w}}
\def\xx{\mathbf{x}}

\def\bB{\mathcal{B}}

\def\sS{\mathcal{S}}
\def\tT{\mathcal{T}}

\def\Ce{\mathbb{C}}

\def\Re{\mathbb{R}}

\def\Ze{\mathbb{Z}}

\DeclareMathOperator*{\argmax}{arg\,max}

\DeclareMathSymbol{@}{\mathord}{letters}{"3B}

\newcommand\paren[1]{\left(#1\right)}

\newcommand\norm[1]{\left\lVert#1\right\rVert}

\def\latex/{\LaTeX}
\def\bibtex/{\hologo{BibTeX}}

\newcommand{\RN}[1]{%
  \textup{\uppercase\expandafter{\romannumeral#1}}%
}
\def\FFF{%
	\textbf{\textup{\uppercase\expandafter{\romannumeral 1}}}%
}
\def\SFF{
	\textbf{\textup{\uppercase\expandafter{\romannumeral 2}}}%
}

\def\oomega{\boldsymbol{\omega}}

\usepackage{booktabs}
\usepackage{xfrac}

\newtheorem{theorem}{Theorem}

\newtheorem{definition}{Definition}[section]
\newtheorem{lemma}[theorem]{Lemma}

\usepackage[capitalize]{cleveref}
\crefname{section}{Sec.}{Secs.}
\Crefname{section}{Section}{Sections}
\Crefname{table}{Table}{Tables}
\crefname{table}{Tab.}{Tabs.}

\newcommand{\bacon}{BACON}

\newif\ifdraft
\drafttrue

\ifdraft
\newcommand{\todo}[1]{\textcolor{red}{todo:#1}}
\newcommand{\guandao}[1]{\textcolor{blue}{guandao:#1}}
\newcommand{\sagie}[1]{\textcolor{orange}{sagie:#1}}
\newcommand{\varun}[1]{\textcolor{brown}{VJ:#1}}
\newcommand{\barron}[1]{\textcolor{magenta}{barron:#1}}
\newcommand{\bh}[1]{\textcolor{cyan}{BH:#1}}
\newcommand{\tom}[1]{\textcolor{green}{TF:#1}}
\newcommand{\serge}[1]{\textcolor{yellow}{SB:#1}}

\else
\newcommand{\todo}[1]{}
\newcommand{\guandao}[1]{}
\newcommand{\sagie}[1]{}
\newcommand{\varun}[1]{}
\newcommand{\barron}[1]{}
\newcommand{\bh}[1]{}
\newcommand{\tom}[1]{}
\newcommand{\serge}[1]{}
\fi

\title{Polynomial Neural Fields \\ for Subband Decomposition and Manipulation}

\makeatletter
\newcommand{\printfnsymbol}[1]{%
  \textsuperscript{\@fnsymbol{#1}}%
}
\makeatother

\author{%
  Guandao Yang\thanks{Equal contribution. Part of this work was done while Guandao was a student researcher at Google.} \\ Cornell University \And
  Sagie Benaim\printfnsymbol{1} \\ University of Copenhagen 
  \And
  Varun Jampani \\ Google Research \And
  Kyle Genova \\ Google Research \And
  Jonathan T. Barron \\ Google Research \And
  Thomas Funkhouser \\ Google Research \And 
  Bharath Hariharan \\ Cornell University \And
  Serge Belongie \\ University of Copenhagen 
}

\begin{document}

\maketitle

\begin{abstract}
Neural fields have emerged as a new paradigm for representing signals, thanks to their ability to do it compactly while being easy to optimize.
In most applications, however, neural fields are treated like black boxes, which precludes many signal manipulation tasks.
In this paper, we propose a new class of neural fields called polynomial neural fields (PNFs). 
The key advantage of a PNF is that it can represent a signal as a composition of a number of manipulable and interpretable components without losing the merits of neural fields representation.
We develop a general theoretical framework to analyze and design PNFs.
We use this framework to design Fourier PNFs, which match state-of-the-art performance in signal representation tasks that use neural fields.
In addition, we empirically demonstrate that Fourier PNFs enable signal manipulation applications such as texture transfer and scale-space interpolation.
Code is available at \url{https://github.com/stevenygd/PNF}.
\end{abstract}

\section{Introduction}\label{sec:intro}

Neural fields are neural networks that take as input spatial coordinates and output a function of the coordinates, such as image colors~\cite{sitzmann2020siren,chen2021learning}, 3D signed distance functions~\cite{park2019deepsdf, atzmon2019sal}, or radiance fields~\cite{mildenhall2020nerf}.
Recent works have shown that such representations are compact~\cite{dupont2021coin, martel2021acorn, saragadam2022miner}, allow sampling at arbitrary locations~\cite{tancik2020fourier,sitzmann2020siren}, and are easy to optimize within the deep learning framework.
These advantages enabled their success in many spatial visual computing applications including novel view synthesis~\cite{sitzmann2019srns,Oechsle2019ICCV,Niemeyer2020CVPR,mildenhall2020nerf,yariv2020multiview} and 3D reconstruction~\cite{eslami2018neural,sitzmann2019srns,jiang2020local,peng2020convolutional,sitzmann2020siren}.
Most recent neural field based methods, however, treat the network as a black box.
As such, one can only obtain information by querying it with spatial coordinates.
This precludes the applications where we want to change the signal represented in a neural field.
For example, it is difficult to remove high frequency noise
or change the stationary texture of an image.

One way to enable signal manipulation is to decompose a signal into a number of interpretable components and use these components for manipulation.
A general approach that works across a wide range of signals is to decompose them into frequency subbands~\cite{adelson1984pyramid, burt1987laplacian, simoncelli1995steerable}.
Such decompositions are studied in the traditional signal processing literature, e.g., the use of Fourier or Wavelet transforms of the spatial signal.
These transformations, however, usually assume the signal is densely sampled in a regular grid.
As a result, it is non-trivial to generalize these approaches to irregular data (e.g., point clouds).
Another shortcoming of these transforms is that they require a lot of terms to represent a signal faithfully, which makes it difficult to scale to signals of more than two dimensions, such as in the case of light fields.
Interestingly, these are the very problems that can be solved by neural fields, which 
can represent irregular signals compactly and scale easily to higher dimensions.
This leads to the central question of this paper: can we incorporate the interpretability and controllability of classical signal processing pipelines to neural fields? 

Our goal is to design a new class of neural fields that allow for precise subband decomposition and manipulation as required by the downstream tasks.
To achieve that, our network needs to have the ability to control different parts of the network to output signal that is limited by desirable subbands. 
At the same time, we want the network to inherit the usual advantages of neural fields, namely, being compact, expressive, and easy to optimize.
The most relevant prior works with related aims are Multiplicative Filter Network (MFN)~\cite{Fathony2021MultiplicativeFN} and BACON~\cite{Lindell2021BACONBC}.
While MFNs enjoy the advantages of neural fields and
 are easy to analyze, they do not use this property to control the network's output for subband decomposition.
BACON~\cite{Lindell2021BACONBC} extends the MFN architecture to enforce that its outputs are upper-band limited, but it
lacks the ability to provide subband control beyond upper band limits.
This hinders BACON's applicability to tasks that requires more precise control of subbands, such as manipulating stationary textures (shown in \cref{sec:texture}).

To address these issues, we propose a novel class of neural fields called \textbf{polynomial neural fields (PNFs)}. 
PNF is a polynomial neural network~\cite{chrysos2019polygan} evaluated using a set of basis functions.
PNFs are compact, easy to optimize, and can be sampled at arbitrary locations as with general neural fields.
Moreover, PNFs enjoy interpretability and controllability of signal processing methods. 
We provide a theoretical framework to analyze the output of PNFs.
We use this framework to design the Fourier PNF, whose outputs can be localized in the frequency domain with both upper and lower band limits, along with orientation specificity.
To the best of our knowledge, this is the first neural field architecture that can achieve such a fine-grained decomposition of a signal.
Empirically, we demonstrate that Fourier PNFs can achieve subband decomposition while reaching state-of-the-art performance for signal representation tasks.
Finally, we demonstrate the utility of Fourier PNFs in signal manipulation tasks such as texture transfer and scale-space interpolation. 

\section{Related Works}

Our method are built on three bodies of prior works: signal processing, neural fields, and polynomial neural networks.
In this section, we will focus on the most relevant part of these prior works. 
For further readings, please refer to \citet{orfanidis1995introduction} for signal processing, \citet{xie2022neural} for neural fields, and \citet{PNNTutorial} for polynomial neural networks.

\noindent{\bf Fourier and Wavelet Transforms.}
In traditional signal processing pipeline, one usually first transforms the signal into weighted sums of functionals from certain basis before manipulating and analyzing the signal~\cite{orfanidis1995introduction}.
The Fourier and Wavelet transformation are most relevant to our work.
In particular, prior works has leveraged Fourier and Wavelet transforms to organize image signal into meaningful and manipulable components such as the Laplacian~\cite{burt1987laplacian} and Steerable~\cite{simoncelli1995steerable} pyramids.
In our paper, we analyze the signal in terms of the basis functions studied by Fourier and Wavelet transforms.
Our manipulable components are also inspired by the subband used in Steerable Pyramid~\cite{simoncelli1991subband}.
While this signal processing pipeline is very interpretable, it's non-trivial to make it work on irregular data because these transformations usually assume the signal to be densely sampled in a regular grid.
In this paper, we tries to combines the interpretability of traditional signal processing pipeline with the merits of neural fields, which is easy to optimize even with irregular data.

\noindent{\bf Neural Fields.} 
Neural Fields are neural networks that maps spatial coordinate to a signal.
Recent research has shown that neural fields are effective in representing a wide variety of signals such as images~\cite{stanley2007compositional,sitzmann2020siren}, 3D shapes~\cite{davies2020effectiveness, michalkiewicz2019implicit, park2019deepsdf, atzmon2019sal,yang2019pointflow}, 3D scenes~\cite{eslami2018neural,sitzmann2019srns,jiang2020local,peng2020convolutional,gropp2020implicit}
and radiance fields~\cite{liu2020neural,jiang2020sdfdiff,liu2020dist,martinbrualla2020nerfw,pumarola2020d,srinivasan2020nerv,zhang2020nerf,neff2021donerf,oechsle2021unisurf,boss2021nerd,boss2021neural,garbin2021fastnerf,lindell2020autoint,wang2021neus,yariv2021volume,yu2020pixelnerf, barron2021mip,saito2019pifu,sitzmann2019srns,Oechsle2019ICCV,Niemeyer2020CVPR,mildenhall2020nerf,yariv2020multiview}. 
However, neural fields typically operate as black boxes, which hinders the application of neural fields to some signal decomposition and manipulation tasks as discussed in \cref{sec:intro}.
Recent works have to alleviate such issue by designing network architecture that are partially interpretable.
A common technique is to encode the input coordinate with a positional encoding where one can control spectrum properties such as the frequency bandwidth inputting into the network~\cite{mildenhall2020nerf,tancik2020fourier,sitzmann2020siren,barron2021mip,zheng2021rethinking}.
But these positional encodings are passed through a black-box neural networks, making it difficult to analyze properties of the final output.
The most relevant works is BACON~\cite{Lindell2021BACONBC}, which propose an initialization schema for multiplicative filter networks (MFNs)~\cite{fathony2020multiplicative} that ensures the output to be upper-limited by certain bandwidth.
This work generalizes BACON in two ways.
First, our theory can be applied to a more general set of basis function and network topologies.
Second, our network enables more precise subband controls, which include band-limiting from above, below, and among certain orientation.

\noindent{\bf Polynomial Neural Networks.}
Polynomial neural networks (PNNs) are generally referred to neural networks composed of polynomials~\cite{PNNTutorial}.
The study of PNNs can be dated back to higher-order boltzmann machine~\cite{sejnowski1986higher} and Mapping Units~\cite{hinton1985shape}.
Recently, research has shown that PNNs can train very strong generative models~\cite{chrysos2020p,chrysos2019polygan,georgopoulos2020multilinear,chrysos2021conditional,wu2022adversarial} and recognition models~\cite{hu2018squeeze, wang2018non}.
The empirical success has also followed with deeper theoretical analysis.
For example, ~\cite{chrysos2020p} reveal how PNNs' architecture relates to polynomial factorization.
\citet{kileel2019expressive} and \citet{choraria2022spectral} studies the expressive power of PNNs.
Our work establish the connection between PNNs and many neural fields such as MFN~\cite{Fathony2021MultiplicativeFN} and BACON~\cite{Lindell2021BACONBC}.
We further extends polynomial neural networks by evaluating the polynomial with a set of basis functions such as the Fourier basis.

\section{Method}
\label{sec:method}

In this section, we will provide a definition for polynomial neural fields (\cref{sec:pnf}). 
From this definition, we derive a theoretical framework to analyze their outputs in terms of subbands (\cref{sec:subband-control}).
We then use this framework to design Fourier PNFs, a novel neural fields architecture to represent signals as a composition of fine-grain subbands in frequency spaces (\cref{sec:ff-pnf}).

\subsection{Polynomial Neural Fields}\label{sec:pnf}

Recall that we would like to maintain the merits of the neural field representation while adding the ability to partition it into analyzable components.
As with all neural networks, to guarantee expressivity, we want to base our neural fields on function compositions~\cite{raghu2016survey}.
At the same time, we want our neural fields to be interpretable in terms of a set of basis functions that have known properties, as in the signal processing literature.
We propose the following class of neural fields:
\begin{definition}[PNF]
Let $\bB$ be a basis for the vector space of functions for $\Re^n \to \Re$.
A \emph{Polynomial neural field of basis $\bB$} is a neural network $f = g_L \circ \dots \circ g_1 \circ \gamma$,
where $\forall i, g_i$ are finite degree multivariate polynomials, 
and $\gamma:\Re^n \to\Re^d$ is a $d$-dimensional feature encoding using basis $\bB$: $\gamma(x) = [\gamma_1(x)\ \dots\ \gamma_d(x)]^T, \gamma_i \in \bB, \forall i$. 
\label{def:pnf}
\end{definition}
This definition allows a rich design space that subsumes several prior works.
For example, MFN~\cite{Fathony2021MultiplicativeFN} and BACON~\cite{Lindell2021BACONBC} can be instantiated by setting $g_i$ to be either the linear layer or a masked multiplication layer.
Similarly, if we set the basis to be $\bB=\{x^n\}_{n\in \mathbb{N}}$, then we can show architectures proposed in $\Pi$-Net~\cite{chrysos2020p} are a subclass of PNF.
This rich design space can potentially allow us to tailor the architecture to the application of interests, as demonstrated later in \cref{sec:ff-pnf} and \cref{sec:scalespace}.
Moreover, such rich space also contains many expressive neural networks, as shown by both the prior works~\cite{Fathony2021MultiplicativeFN,Lindell2021BACONBC,chrysos2019polygan,chrysos2020p}
and by our experiment in \cref{sec:expressivity}.

Furthermore, as long as the span of the basis is closed under multiplication, PNF yields a linear combination of basis functions and is thus easy to analyze:
\begin{theorem}[]
Let $F$ be a PNF with basis $\bB$ s.t 
$\forall b_1, b_2 \in \bB, 
b_1(x)b_2(x) = \sum_{i \in I} a_i b_i(x), |I| < \infty
$. Then the output of $F$ is a finite linear sum of the basis functions from $\bB$.
\label{thm:interp-main}
\end{theorem}
Many commonly used bases, such as Fourier, Gabor, and spherical harmonics, all satisfy this property.
Please refer to the supplementary for proofs for a variety of different bases.

\subsection{Controllable Subband Decomposition}\label{sec:subband-control}
\cref{thm:interp-main} is not enough to control or manipulate the signal represented by the network, because any single neuron in the network may potentially be working with an arbitrary set of basis functions.
This is a problem if we want to manipulate the signal through these neurons.
For example, if we want to discard the contributions of some  high-frequency components in the Fourier basis, we cannot decide which neurons to discard if each of them are contributed to a variety of frequency bands.

To allow better manipulation, we use the notion of \emph{subbands} from traditional frequency domain analysis~\cite{Vetterli1995WaveletsAS}.
In the most general sense, a subband is simply a \emph{subset} of the basis.
Manipulations such as smoothing or sharpening can then be done by discarding or enhancing the contributions of one or more such subbands.
One way to instantiate these ideas with a PNF is to represent the signal as a sum of different PFNs, each of which is limited to only a specific subband; then one can manipulate these component PFNs separately.
\begin{definition}[]
A PNF $F$ of basis $\bB$ is limited by a  subband $S \subset \bB$ if $F$ is in the span of  $S$.
\label{def:subband-limited}
\end{definition}

A key challenge is to construct a subband limited PNF for certain subbands.
To this end, we need to understand how the subband-limited PNF transforms under different network operations.
Fortunately, for PNFs we only need to study two types of operations: multiplication and addition.
To this end, we need the notion of a \emph{PNF-controllable set of subbands}:

\begin{definition}[PNF-controllable Set of Subbands]
$\sS = \{S_\theta | S_\theta \subset \bB\}_\theta$ is a PNF-Controllable Set of Subbands for basis $\bB$ if (1) $S_{\theta_1} \cup S_{\theta_2} \in \sS$ and (2) there exists a binary function $\otimes: \sS \times \sS \to \sS$ such that 
if $b_1\in S_{\theta_1}, b_2\in S_{\theta_2} \implies b_1 b_2 \in S_{\theta_1} \otimes S_{\theta_2}$.
\label{def:controllable-subband-decomposition-main}
\footnote{It's possible prove \cref{thm:controllability-main} with a more relax version of \cref{def:controllable-subband-decomposition-main}: $b_1 b_2 \in Span(S_{\theta_1} \otimes S_{\theta_2})$.}
\end{definition}

Intuitively, a PNF-controllable set of subbands lends some \emph{predictability} to what happens if two PNFs, limited to different subbands, are put together into a larger PNF:

\begin{theorem}[]
Let $\sS$ be a PNF-controllable set of subbands of basis $\bB$ with binary relation $\otimes$.
Suppose  $F_1$ and $F_2$ are polynomial neural fields of basis $\bB$ that maps $\Re^n$ to $\Re^m$.
Furthermore, suppose $F_1$ and $F_2$ are subband limited by $R_1\in \sS$ and $R_2\in \sS$.
Then we have the following:
\begin{enumerate}
    \item $W_1F_1(x)+W_2F_2(x)$ is a PNF of $\bB$ limited by subband $R_1 \cup R_2$ with $W_1, W_2 \in \Re^{m\times n}$;
    \item $F_1(x)^TWF_2(x)$ is a PNF of $\bB$ limited by subband $R_1 \otimes R_2$ with $W \in \Re^{m \times m}$.
\end{enumerate}
\label{thm:controllability-main}
\end{theorem}

Many structures used in the signal processing literature, such as the Steerable Pyramid, use subbands that are PNF-controllable~\cite{simoncelli1995steerable, burt1987laplacian}.
Please refer to the supplementary for the derivations of PNF-controllable sets of subbands. In the following sections, we will focus on using Fourier bases and show how to build PNFs that instantiate subband manipulation efficiently.

\subsection{Fourier PNF}\label{sec:ff-pnf}

In this section, we demonstrate how to build a PNF that can decompose a signal into frequency subbands in the following three steps.
First, we identify a PNF-controllable sets of subband for the Fourier basis.
Second, we choose a finite collection of subbands for the PNF to output, and organize them into controllable sets.
The final step is to instantiate the PNF compactly using \cref{thm:controllability-main}.

\subsubsection{Controllable Subband Decomposition for Fourier Space}
\begin{figure*}[h]
    \centering
    \begin{tabular}{@{}c@{}c@{}c@{}c|c@{}c@{}c@{}c@{}}
        \textcolor{red}{$R^{(1)}$} &
        \textcolor{orange}{$R_1$} &
        \textcolor{blue}{$R_2$} &
        \textcolor{green}{$R_1\otimes R_2$} &
        \textcolor{red}{$R^{(\infty)}$} &
        \textcolor{orange}{$R_3$} &
        \textcolor{blue}{$R_4$} &
        \textcolor{green}{$R_3 \otimes R_4$} 
        \\
         \includegraphics[width=0.12\linewidth, height=0.12\linewidth]{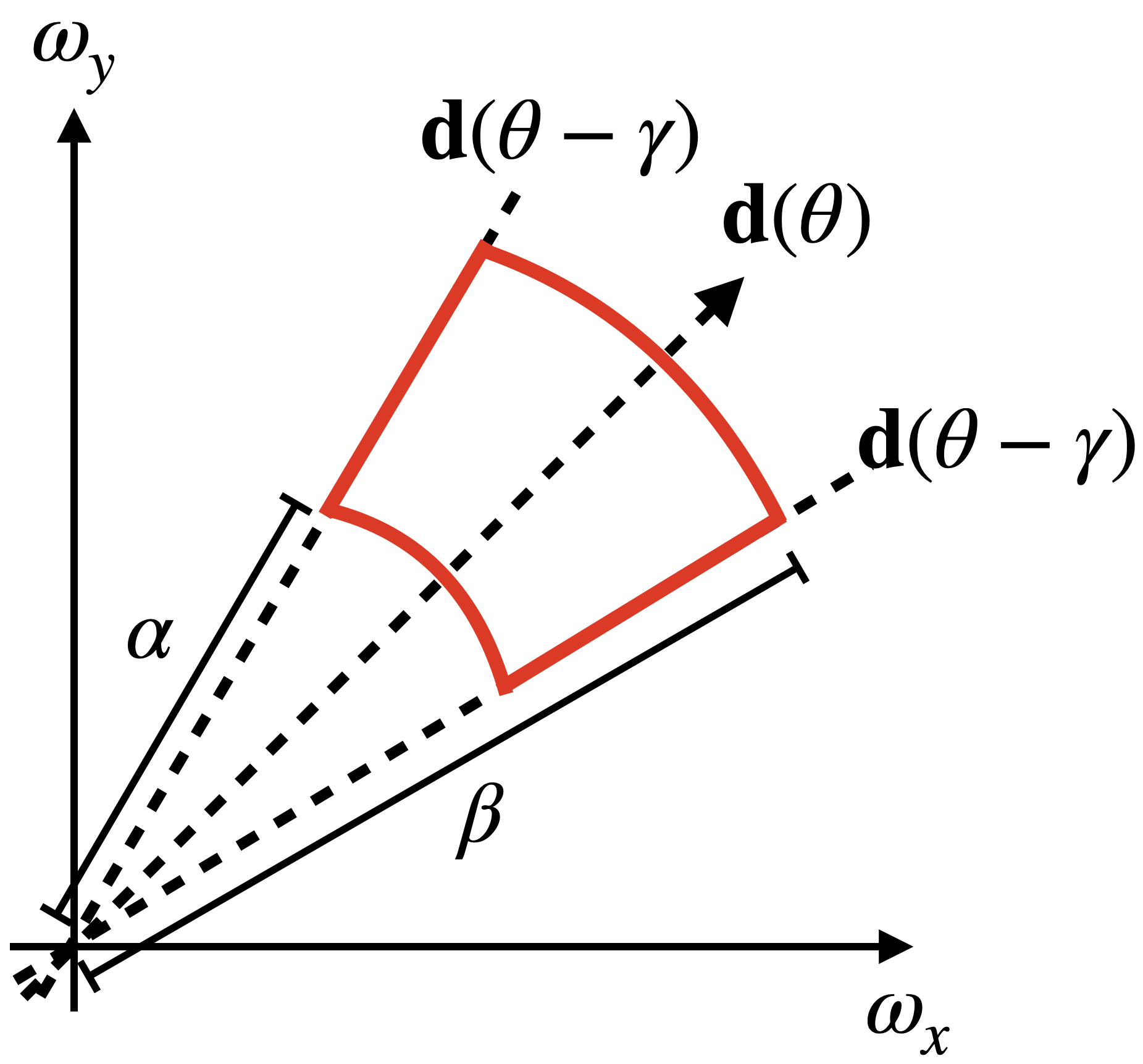} &
         \includegraphics[width=0.12\linewidth, height=0.12\linewidth]{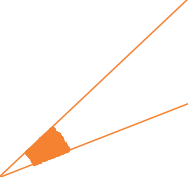} &
         \includegraphics[width=0.12\linewidth, height=0.12\linewidth]{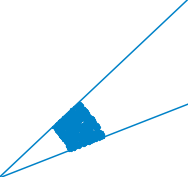} &
         \includegraphics[width=0.12\linewidth, height=0.12\linewidth]{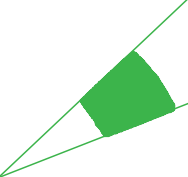} &
         \includegraphics[width=0.12\linewidth, height=0.12\linewidth]{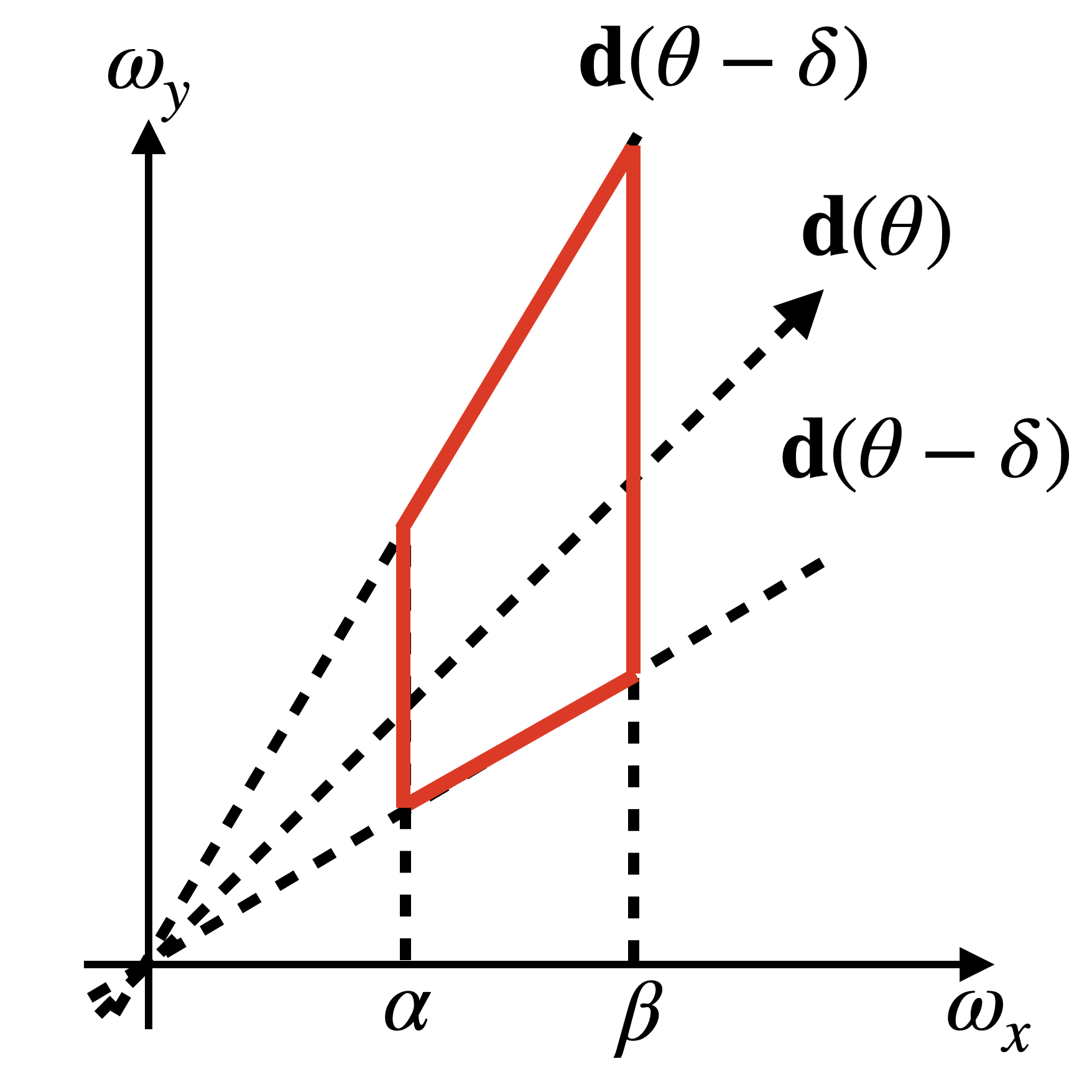} &
         \includegraphics[width=0.12\linewidth, height=0.12\linewidth]{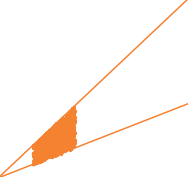} &
         \includegraphics[width=0.12\linewidth, height=0.12\linewidth]{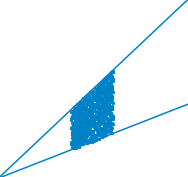} &
         \includegraphics[width=0.12\linewidth, height=0.12\linewidth]{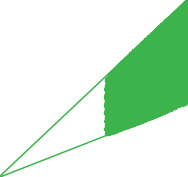}
    \end{tabular}
    \caption{Illustration of PNF-controllable subbands in Fourier space. Left: 2-norm series of subbands; Right: infinite-norm series of subbands.}
    \label{fig:controllable-subbands}
\end{figure*}

The Fourier basis can be written as:
$b_{\oomega}(\xx) = \exp{(i \oomega^T\xx)}$,
where $\oomega, \xx \in \Re^d$ and $i$ is the imaginary unit.
It's easy to see that the fourier basis is complete under multiplication, which satisfies the condition for \cref{thm:interp-main}:
$\exp(i\oomega_1^T\xx)\exp(i\oomega_2^T\xx) = \exp(i(\oomega_1 + \oomega_2)^T\xx).$

Now we need to divide frequency space of $\oomega \in \Re^d$  into subbands that are easy to manipulate and meaningful for downstream tasks.
We define the subband following \citet{simoncelli1995steerable}.
Formally, frequency space is decomposed into following sectors:
\begin{align}
R^{(p)}(\alpha, \beta, \dd, \gamma) =
    \left\{
    \oomega |
    \alpha \leq \norm{\oomega}_p \leq \beta, \norm{\dd}=1,
    \oomega^T\dd < \gamma\norm{\oomega}
    \right\}
    \label{def:ff-subband},
\end{align}
where $p$ describe which norm do we choose to describe the frequency bands.
Intuitively, $\alpha$ defines the lower band limits and $\beta$ defines the upper band limits. 
The vector $\dd$ defines the orientation of the subband and $\gamma$ defines the angular width of the subband.
\cref{fig:controllable-subbands} provides illustrations of $R^{(p)}$.

This definition of subband allows us to organize them into controllable sets.
For example, we can show that the following sets of subbands are controllable:
\begin{theorem}[]
Let $\mathcal{S}$
be a set of subbands defined as $\mathcal{S} = \{R^{(2)}(\alpha, \beta, \dd, \gamma) | \forall  0 \leq \alpha \leq \beta\}$.
If $|\gamma| < \frac{\pi}{4}$, then $\mathcal{S}$ is PNF-controllable.
Specifically,
if $\oomega_1 \in R^{(2)}_1(\alpha_1, \beta_1) \in \mathcal{S}$, and 
$\oomega_2 \in R^{(2)}_2(\alpha_2, \beta_2) \in \mathcal{S}$, then $b_{\oomega_i} b_{\oomega_i} = b_{\oomega_3}$ implies that
$\oomega_3 \in R^{(2)}(\sqrt{\cos(2\gamma)}(\alpha_1+\alpha_2), \beta_1 + \beta_2) \in \mathcal{S}$.
\label{thm:ff-control-main}
\end{theorem}
This theorem captures the intuition that the multiplication of two waves of similar orientations creates high-frequency waves at that orientation.
It allows us to predict the spectrum properties of the output of the network when knowing the spectrum properties of the inputs.
\cref{fig:controllable-subbands} provides illustrations of how two 2D subbands interact under multiplications.

\subsubsection{Subband Tiling}\label{sec:tiling}
In order to represent different signals, our networks need to be able to leverage basis functions with different orientations and within different bandwidths.
With that said, we need to choose a set of PNF-controllable subbands to cover all basis functions we want to use.
For example, we can tile the space with the set of controllable subbands in \cref{thm:ff-control-main} in the following way:
\begin{align}
\tT_{circ} = \{
S_{ij} = R^{(2)}(b_i, b_{i+1}, \dd(\theta_j), \delta) | b_1\leq \dots \leq b_n, \theta_j = j\delta,\delta=\frac{\pi}{m},1\leq j \leq 2m\},
\label{eq:2d-tile-polar}
\end{align}
where $\dd(\theta) = [\sin(\theta), \cos(\theta)]^T$ denotes unit vector rotate with angle $\theta$, $\delta=\frac{\pi}{m}$, and $m$ sufficiently small to allow the application of ~\cref{thm:ff-control-main}.
\begin{wrapfigure}{R}{0.4\textwidth}
\vspace{-3em}
\begin{center}
\includegraphics[width=0.19\textwidth,height=0.19\textwidth]{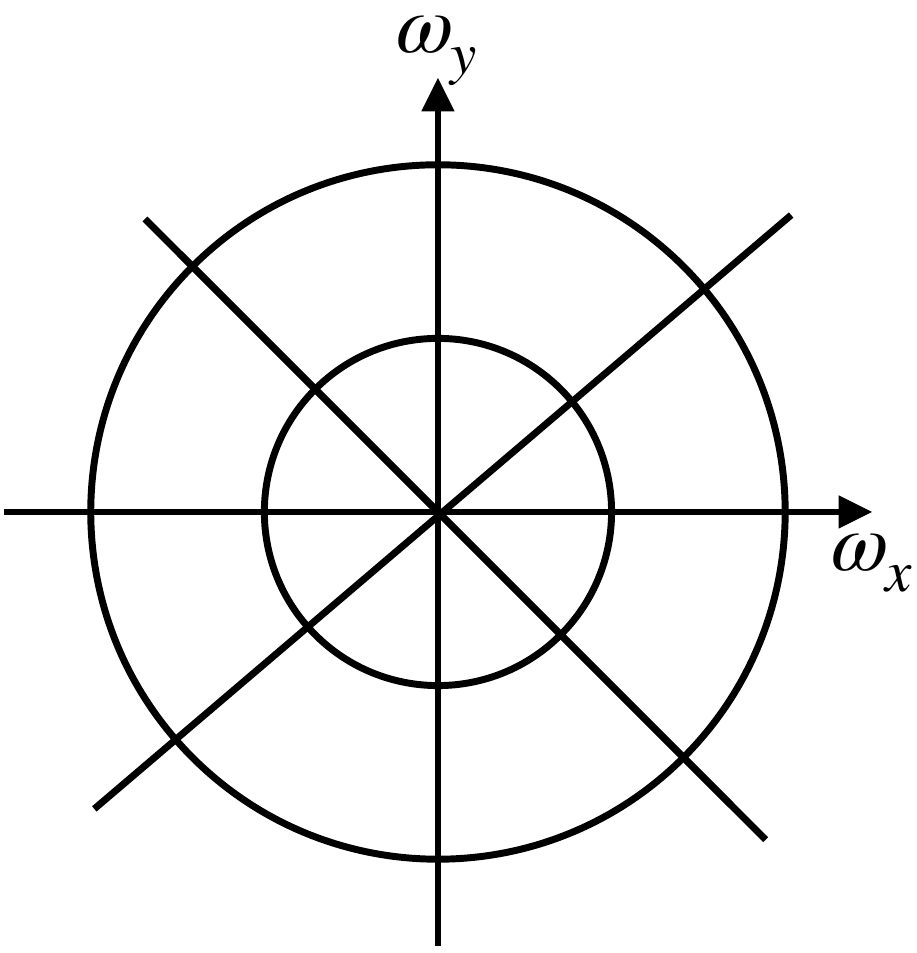}
\includegraphics[width=0.19\textwidth,height=0.19\textwidth]{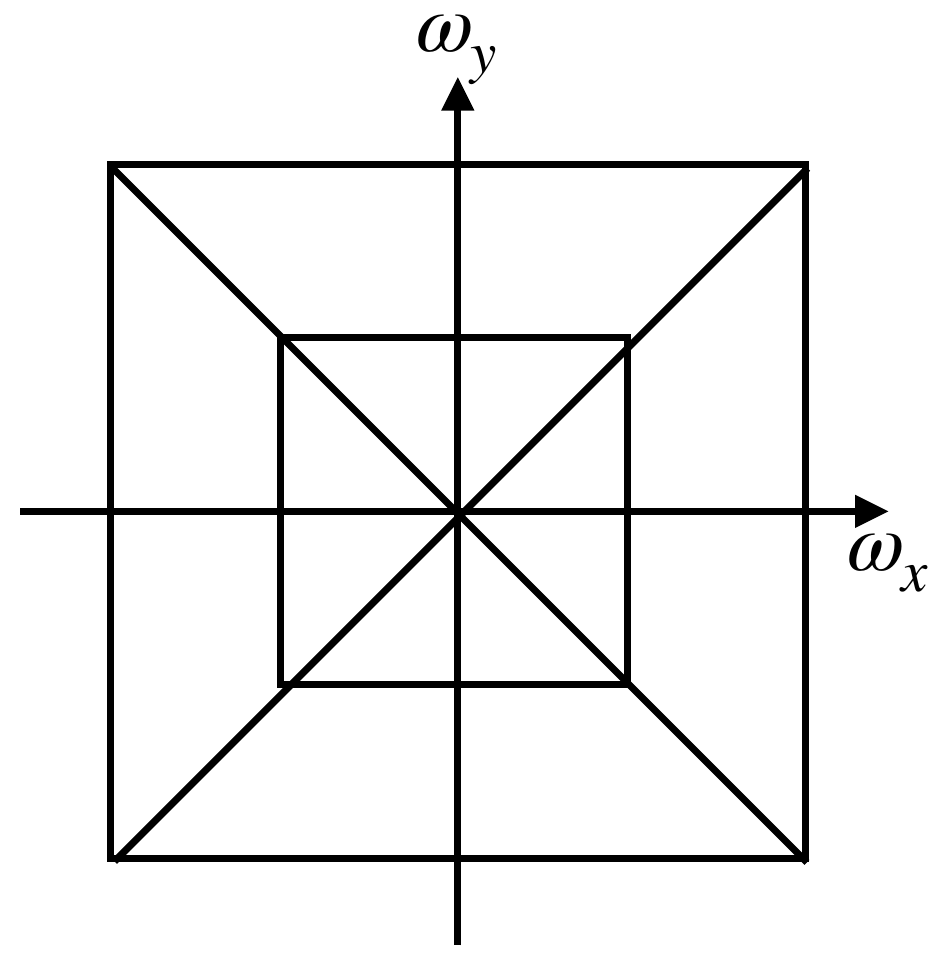}
\caption{Left: $\mathcal{T}_{circ}$; Right: $\mathcal{T}_{rect}$.}
\label{fig:tiling}
\end{center}
\vspace{-3em}
\end{wrapfigure}

For 2D images, the region of interest is $[-N,N]^2$ where $N$ is the bandwidth determined by the Nyquist Sampling Theorem~\cite{nyquist1928certain}.
To tile this rectangular region well without introducing unnecessary high-frequency details, we will use pesudo polar coordinate grids, which divide the space into vertical or horizontal sub-regions and tile those subregions according to $l$-$\infty$ norm.
Formally, the 2D pseudo polar coordinate tiling can be written as:
\begin{align}
\tT_{rect} = \{
S_{ij} = R^{(\infty)}(b_i, b_{i+1}, \dd(\theta_j), \delta) | 
b_1\leq \dots \leq b_n, \theta_j = j\delta,1\leq j \leq 2m, j\neq m\}.
\label{eq:2d-tile}
\end{align}
Note that we exclude certain regions to avoid having a subband to include orientation at $\frac{\pi}{4}$ and $\frac{3\pi}{4}$ in order for \cref{thm:ff-control-main} to generalize to such tiling.
\cref{fig:tiling} contains an illustration of these two types of tiling.
We show detailed derivation in the supplementary.
  
Tiling the spectrum space with subbands from $\tT_{circ}$ or $\tT_{rect}$ allows us to organize information in a variety of meaningful ways.
For example, this set of subbands can be grouped into different cones $\{S_{ij}\}_{j=1}^{2m}$.
Each cone corresponds to a particular orientation of the signal.
Alternatively, we can also organize this set of subbands into different rings $\{S_{ij}\}_{i=1}^n$, which corresponds to the decomposition of an image into the Laplacian Pyramid.

\subsubsection{Network architecture} 
\begin{figure*}[ht]
    \centering
     \includegraphics[width=\linewidth]{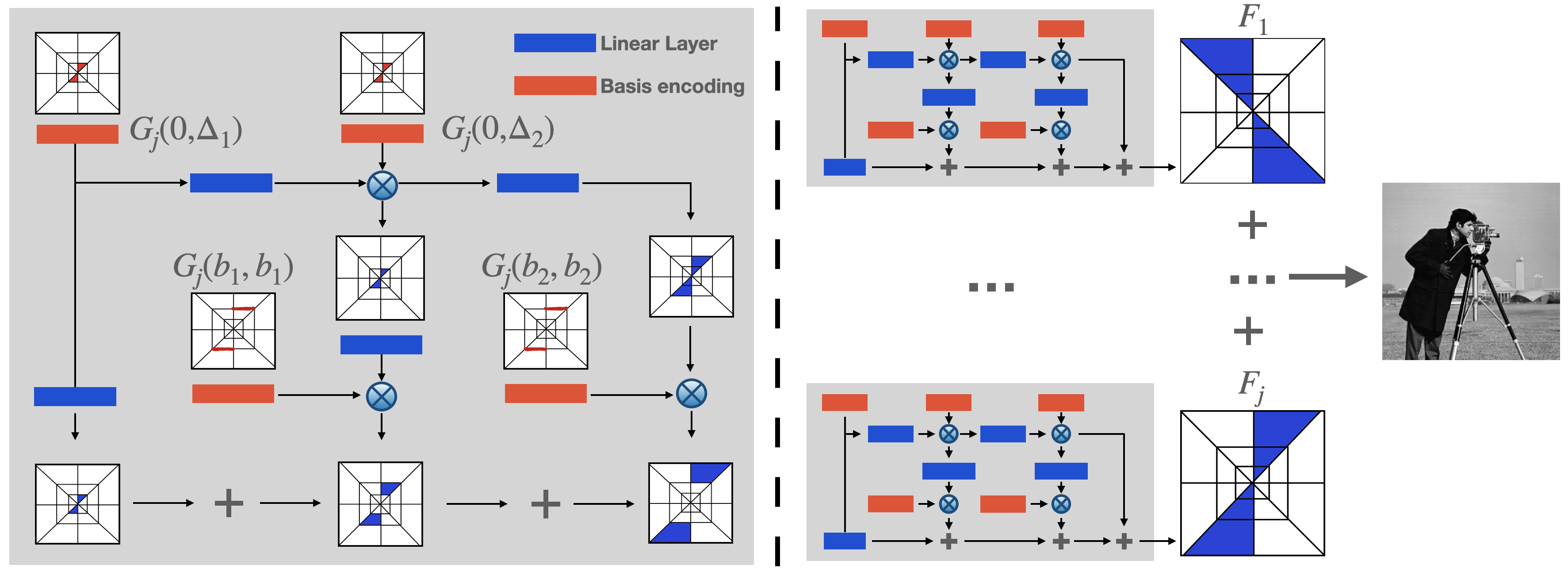} 
     \caption{Illustration of Fourier PNF architecture. Fourier PNF is an ensemble model. The final result is summed over a series of PNFs $F_j$, whose structure is shown on the left side of the figure.}
    \label{fig:architecture}
\end{figure*}
We now have decided on the set of subbands we want to produce by PNF.
We want to design the final Fourier PNF as an ensemble of subband limited PNFs:
$F(\xx) = \sum_j\sum_i O_{ij}F_{ij}(\xx),$
where $F_{ij}(\xx): \Re^n\to\Re^h$ is a PNF that's limited with subband $S_{ij}$ defined in \cref{eq:2d-tile}, and $O_{ij}\in \Re^{h\times m}$ aggregates the output signals together.
One way to achieve this is to naively define $F_{ij}$ as a two layers PNF the feature encoding layer to include only the basis functions in the subband $S_{ij}$ followed by a linear layer.
However, such an approach fails to achieve good performance without a huge number of trainable parameters.
Alternatively, we leverage \cref{thm:controllability-main} to factorize $F$:
\begin{align}{\textstyle
F(\xx) = \sum_j F_j(\xx),\ &F_j(\xx) = \sum_{k=1}^nG_{j}(\xx, b_i, b_i) W_{jk} Z_{j,k}(\xx), \label{eq:ff-pnf-within-set} \\
    Z_{j,1}(\xx) = G_{j}(\xx, 0, \Delta_1),\ &
    Z_{j,k}(\xx) = G_{j}(\xx, 0, \Delta_k)W_i Z_{j,k-1}(\xx),
}\end{align}
where $G_{j}(\xx, a, b)$ is subband limited in $R^{(\infty)}(a, b, d(\theta_j), \delta)$ and $\Delta_k = b_k - b_{k-1}$.
This network architecture of $F_i$ is illustrated in \cref{fig:architecture}.
We instantiate this architecture by setting $G_{j}(\xx, a, b)$ into a linear transform of basis sampled from the subband to be limited:
\begin{align}
    G_{j}(\xx, a, b) = W_i\gamma_j(\xx), \gamma_j \in R^{(\infty)}(a, b, d(\theta_j), \delta)^d, W_i\in \Re^{h\times d},
\end{align}
where $h$ and $d$ is the dimension for the output and the feature encoding.
We provide additional implementation details in the supplementary.

\section{Results}
\label{sec:experiments}

We demonstrate the applicability of our framework along three different axes: (\textit{Expressivity}) For a number of different signal types, we demonstrate our ability to fit a given signal. We observe that our method enjoys faster convergence in terms of the number of iterations. (\textit{Interpretability}) 
We visually demonstrate our learned PNF, whose outputs are localized in the frequency domain with both upper band-limit and lower-band limit. 
(\textit{Decomposition}) Finally, we demonstrate the ability to control a PNF on the tasks of texture transfer and scale-space representation, based on our  subband decomposition. For all presented experiments, full training details and 
additional qualitative and quantitative results are provided in the supplementary.

\subsection{Expressivity}\label{sec:expressivity}

\newcolumntype{Y}{>{\centering\arraybackslash}X}
\begin{figure*}
\begin{center}
\begin{tabular}{c@{}c@{}c@{}c@{~~~}c@{}c@{}c@{}c@{}c}
\includegraphics[width=0.12\linewidth,height=0.12\linewidth]{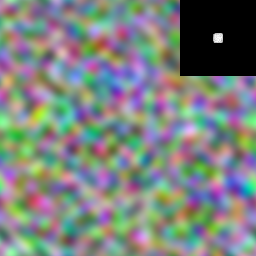} &
\includegraphics[width=0.12\linewidth,height=0.12\linewidth]{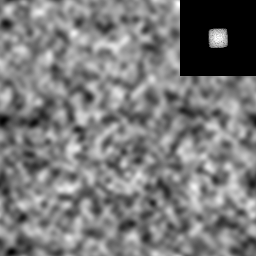} &
\includegraphics[width=0.12\linewidth,height=0.12\linewidth]{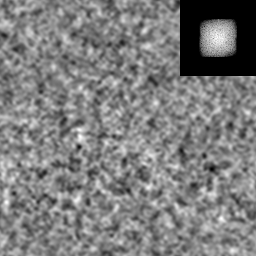} &
\includegraphics[width=0.12\linewidth,height=0.12\linewidth]{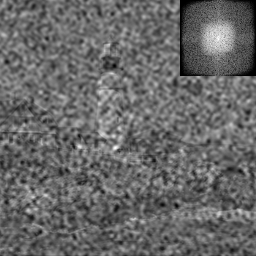} &
\includegraphics[width=0.12\linewidth,height=0.12\linewidth]{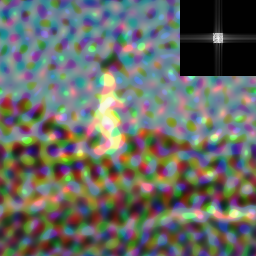} &
\includegraphics[width=0.12\linewidth,height=0.12\linewidth]{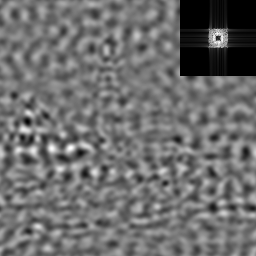} &
\includegraphics[width=0.12\linewidth,height=0.12\linewidth]{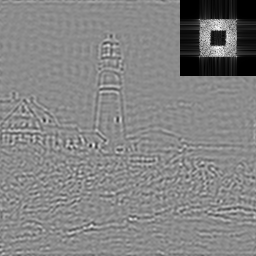} &
\includegraphics[width=0.12\linewidth,height=0.12\linewidth]{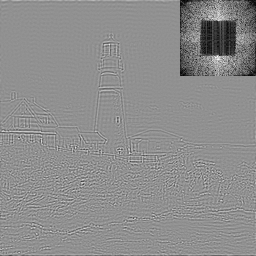} \\

\includegraphics[width=0.12\linewidth,height=0.12\linewidth]{figures/laplacian/bacon/21_BACON_log_quant_last_gau_0.png} &
\includegraphics[width=0.12\linewidth,height=0.12\linewidth]{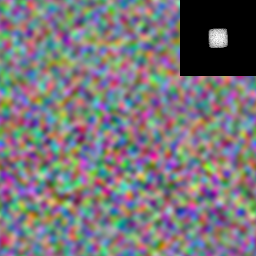} &
\includegraphics[width=0.12\linewidth,height=0.12\linewidth]{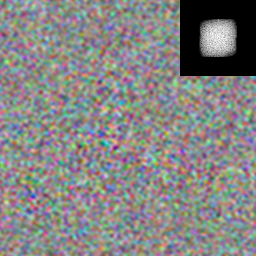} &
\includegraphics[width=0.12\linewidth,height=0.12\linewidth]{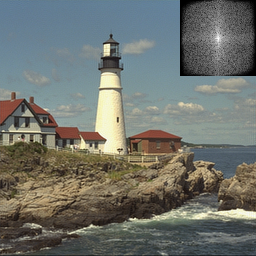} &
\includegraphics[width=0.12\linewidth,height=0.12\linewidth]{figures/laplacian/pnf/21_pnf_PNF_pseudopolar_fix2_gau_0.png} &
\includegraphics[width=0.12\linewidth,height=0.12\linewidth]{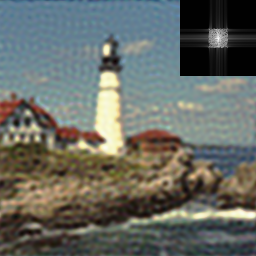} &
\includegraphics[width=0.12\linewidth,height=0.12\linewidth]{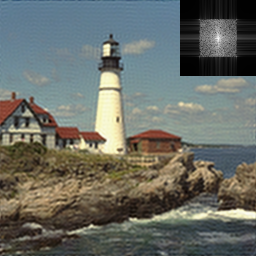} &
\includegraphics[width=0.12\linewidth,height=0.12\linewidth]{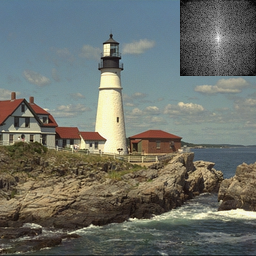}
 \\
 \multicolumn{4}{c}{BACON} & \multicolumn{4}{c}{PNF}
\end{tabular}
\caption{BACON: The bottom row shows the output of each layer which is upper band limited. The top row (columns 2-4) shows the difference between the output of a given layer and the one before it. PNF: The top row shows the output of each layer which is both upper and lower band limited. The bottom row (columns 2-4) shows the addition of the output at a given layer and the one before it. }
\label{fig:laplacian-pyramid}
\end{center}
\end{figure*}

\begin{table}
\centering
\begin{minipage}[t]{.49\textwidth}
\centering
\caption{Image Fitting on the DIV2K dataset.  }
\begin{tabular}{@{}llll@{}}
\toprule
           & PSNR  & SSIM  & \# Params  \\ \midrule
RFF        & 28.72 & 0.834 & 0.26M \\
SIREN      & 29.22 & 0.866 & 0.26M \\
BACON      & 28.67 & 0.838 & 0.27M \\
BACON-L    & 29.44 & 0.871 & 0.27M \\
BACON-M    & 29.44 & 0.871 & 0.27M  \\
PNF        & \textbf{29.47} & \textbf{0.874} & 0.28M  \\ 
\bottomrule
\label{tab:img_fit_quant}
\end{tabular}
\end{minipage}%
\hspace{0.05cm}
\begin{minipage}[t]{.49\textwidth}
\centering
\caption{3D shape fitting. CD is Chamfer Distance ($\times 10^{6}$)}
\begin{tabular}{@{}lccc@{}}
\toprule
              &   CD   & F-score  & \# Params \\ \midrule
SIREN            & 9.00        & 99.76\% & 0.53M \\
BACON            & 2.60        & 99.84\% & 0.54M \\
BACON-L       & 2.60        & 99.85\% & 0.54M \\
BACON-M       & 2.61        & 99.85\% & 0.54M \\            \\
PNF    & \textbf{2.25}	 &  \textbf{99.97\%}           & 0.59M        \\
\bottomrule
\label{tab:sdf}
\end{tabular}
\end{minipage}%
\hspace{0.05cm}
\end{table}

\begin{table}
\centering

\caption{NeRF Fitting for $64^2$ resolution. A comparison between PNF and bacon is shown for PSNR and SSIM at 300 and 500 epochs. as well as corresponding number of parameters used. 
}
\begin{tabular}{@{}llcccc@{~~~~}c}
\toprule
& & 1x & 1/2x & 1/4x & 1/8x & Avg \\ 
\midrule

  300 epochs & BACON & 28.07/0.936 & 29.95/0.943 & 30.75/\textbf{0.939} & 31.37/\textbf{0.927} & 30.04/\textbf{0.936}   \\
 PSNR/SSIM   & PNF &  \textbf{29.89}/\textbf{0.937} & \textbf{31.16}/\textbf{0.946} & 
\textbf{32.11}/0.934 & \textbf{32.00}/0.920 & \textbf{31.29}/0.934 \\ 
 \midrule
 
  500 epochs & BACON & 28.51/0.932 & 30.77/0.941 & \textbf{31.74}/\textbf{0.940} & 
 \textbf{32.10}/\textbf{0.928} & 30.78/0.935\\  
 PSNR/SSIM  & PNF & \textbf{29.33}/\textbf{0.950} & \textbf{31.19}/\textbf{0.950} & 
 31.48/0.936 & 31.37/0.926 & \textbf{30.84}/\textbf{0.940} \\
 \midrule 
 \multirow{2}{*}{\# Params} & BACON & 0.54M & 0.41M & 0.27M & 0.14M & 0.34M \\ 
  & PNF & 0.46M & 0.34M  & 0.23M & 0.12M & 0.29M\\
                      \bottomrule
\label{tab:nerf-expressivity}
\end{tabular}

\end{table}

We demonstrate that a PNF is capable of representing signals of different modalities including images, 3D signed distance fields, and radiance fields.
We compare our Fourier PNF with state-of-the-art neural field representations such as BACON~\cite{Lindell2021BACONBC}, SIREN~\cite{sitzmann2020siren}, and Random Fourier Features~\cite{tancik2020fourier}.

\paragraph{Images} 
Following BACON~\cite{Lindell2021BACONBC}, we train a PNF and the baselines to fit images from the DIV2K~\cite{agustsson2017ntire} dataset.
During training, images are downsampled to $256^2$. All networks are trained for $5000$ iterations.
At test time, we sample the fields at $512^2$ and compare with the original resolution images.

Different from other networks, BACON is supervised with the training image in all output layers.
For fair comparison, we also include two BACON variants, which are only supervised either at the last output layer (``BAC-L'') or using the average of all output layers (``BAC-M'').
We report the PSNR and Structural Similarity (SSIM) scores of these methods in \cref{tab:img_fit_quant}.
Fourier PNF improves on the performance of the previous state-of-the-arts.

In addition to its expressivity, PNF is also able to localize signals in different regions.
Specifically, Fourier PNF decomposes an image into subbands in the frequency domain. 
For example, \cref{fig:laplacian-pyramid} shows that the output branches of the Fourier PNF correspond to specific frequency bands. Note that such decomposition forms for all layers while training is only performed for the last layer of the PNF. Such subband control is reminiscent of the traditional signal processing analogue of Laplacian Pyramids~\cite{burt1987laplacian}.
Note that each layer of the PNF is guaranteed to be both lower and upper frequency band limted, 
while this is not achievable by BACON. For comparison \cref{fig:laplacian-pyramid} shows the corresponding result of training BACON only for the last layer. 
In \cref{sec:texture} and \cref{sec:scalespace} we demonstrate how to leverage such fine-grain localization for texture transfer and scale space interpolation.

\paragraph{3D Signed Distance Field}

\begin{figure}
    \vspace{-0.2cm}
    \includegraphics[width=0.6\linewidth]{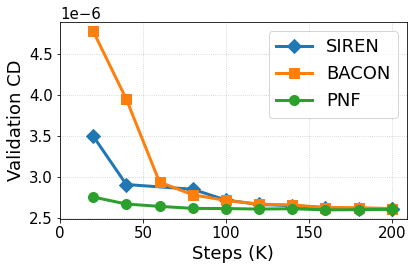}
    \includegraphics[width=0.39\linewidth]{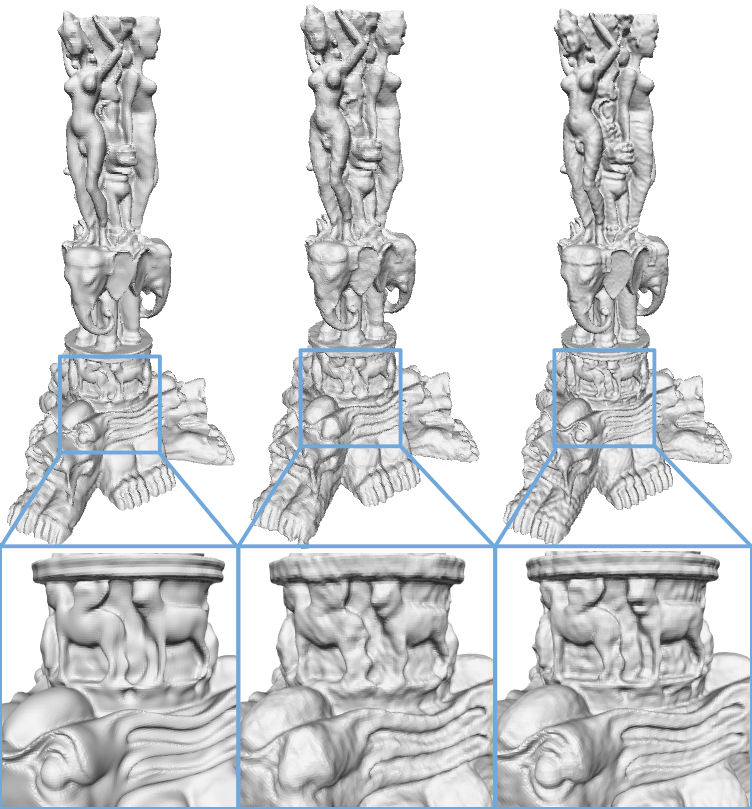}
    \caption{LHS: Convergence time in terms of number of steps/iterations for the \textit{Thai Statue} model~\cite{Stanford3DScanningRepository}. The x-axis shows the number of steps (K) and the y-axis shows the validation Chamfer Distance ($\times 10^6$). RHS: Qualitative comparison for the \textit{Thai Statue} model~\cite{Stanford3DScanningRepository}.}
    \label{fig:thai_full}
\end{figure}

One advantage of neural fields is their ability to represent irregular data such as 3D point cloud or signed distance fields with high fidelity.
We demonstrate that PNFs can represent 3D shapes via signed distance fields expressively.
We follow \bacon{}'s experimental setting to fit a range of 3D shapes from the Stanford 3D scanning repository~\cite{Stanford3DScanningRepository} (a slightly different normalization is used, see supplementary for details).
During training, we sample $10k$ oriented points from the ground truth surface and perturb them with noise to compute an estimate of SDF as ground truth.
All models are trained with reconstruction loss for the same number of iterations.
A quantitative comparison of the different methods is reported in \cref{tab:sdf}. We observe that PNF achieves slightly higher fitting quality than other methods. 
Next, we investigate convergence behavior in \cref{fig:thai_full} (Left) on the \textit{Thai Statue} model~\cite{Stanford3DScanningRepository}. Here, we observe that PNFs converge much more quickly than SIREN and BACON while achieving comparable quality. In \cref{fig:thai_full} (Right) we provide a qualitative comparison of our result, \bacon{}, and SIREN.

\begin{figure}
\begin{center}
\begin{tabular}{cc}
\includegraphics[draft=false,width=0.96\linewidth,height=0.186\linewidth]{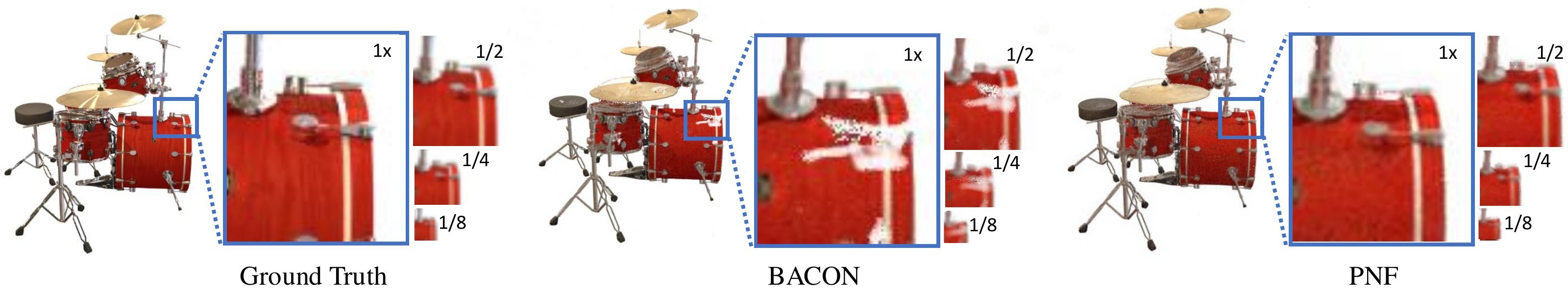}  
\end{tabular}
\caption{Qualitative comparison of neural radiance fields for the \textit{Drums} scene given in \cite{barron2021mip} at the 300 epochs.  The results suggests that our model can achieve better quality in early iterations.}
\label{fig:sdf_nerf}
\end{center}
\end{figure}

\paragraph{Neural Radiance Field}
In the context of neural radiance fields, we show that PNFs can represent signals in higher dimensions compactly and faithfully.
We follow BACON's setting and train a PNF to model the radiance field of a set of Blender scenes~\cite{mildenhall2020nerf}, see supplementary for details. 
Specifically, the PNF outputs a 4D vector of RGB and density values.
These values are used by a volumetric renderer proposed of NeRF~\cite{mildenhall2020nerf} to produce an image, which is supervised with  reconstruction loss.
At test time, we use the same volumetric renderer to produce images from different camera poses and evaluate them with the known ground truth images.
The results, in comparison to BACON, are presented in \cref{tab:nerf-expressivity}. A Qualitative comparison is also given in \cref{fig:sdf_nerf}.
Similarly to SDF, we are able to match or improve BACON's performance using about $60\%$ of the total number of iterations.
This can be potentially attributed to PNF's ability to disentangle coarse signal from finer one; and so at higher layers, BACON needs to relearn low-frequency details in the image while the PNF does not.

\paragraph{Parameter Efficiency.}
One major advantage of using neural fields is that it can represent signal with high expressivity while remaining compact.
In this section, we will demonstrate that PNF also retains this advantages.
While choosing the hyperparameters for PNFs (e.g. hidden layer size) for the expressivity experiemnts, we make sure the PNF has a comparable number of parameters with the prior works.
A comparison of the number of trainable parameters used for different model is included in \cref{tab:img_fit_quant}-\cref{tab:nerf-expressivity}.
We can see that PNFs are achieving comparable performance with the state-of-the-arts neural fields with roughly the same amount of trainable parameters.

\begin{table}[h]
\centering
\centering
\caption{A comparison of the training and inference time for image fitting of a cameramen image. 
}
\begin{tabular}{lcccc}
\toprule
  & Time(s)/Step & Time(s) to 36 PSNR & Final PSNR & Final SSIM  \\
\midrule
 BACON & 0.16 & 177 & 37.45 & 97.33 \\
 PNF & 0.64  & 96 & 37.45 & 97.44 \\
 SIREN & 0.10 & 163 & 36.90 & 97.50 \\
 RFF & 0.08 & 275 & 36.23 & 95.05 \\
\bottomrule
\label{tab:time}
\end{tabular}

\end{table}
\paragraph{Training and Inference Time}
While we observe that Fourier PNF can converge in fewer iterations, but due to the ensemble nature of Fourier PNF, each forward and backward pass of PNF requires longer time to evaluate.
This leads to the question, can we actually achieve faster convergence in wall time?
We profile the training and inference time of the image fitting experiment in \cref{tab:time} for the camera men image. 
The results show that it's possible for PNF to converge faster even in terms of clock time because PNFs can converge is drastically fewer number of steps.

\subsection{Texture Transfer}\label{sec:texture}

\begin{figure*}
\begin{center}
\begin{tabular}{c@{~~}lc@{~~}c@{}c@{}cc@{~~}c@{}c@{}c@{}}

& & & L-1 & L-2 & L-3 & & B-1 & B-2 & B-3 \\  
\rotatebox[origin=l]{90}{~~Content} & \includegraphics[width=0.118\linewidth]{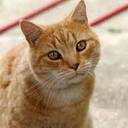} &
\rotatebox[origin=l]{90}{~BACON} & 
\includegraphics[width=0.118\linewidth]{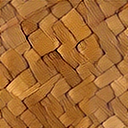} &
\includegraphics[width=0.118\linewidth]{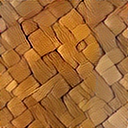} &
\includegraphics[width=0.118\linewidth]{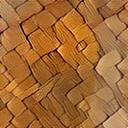} & 
\rotatebox[origin=l]{90}{CLIPStyler} & 
\includegraphics[width=0.118\linewidth]{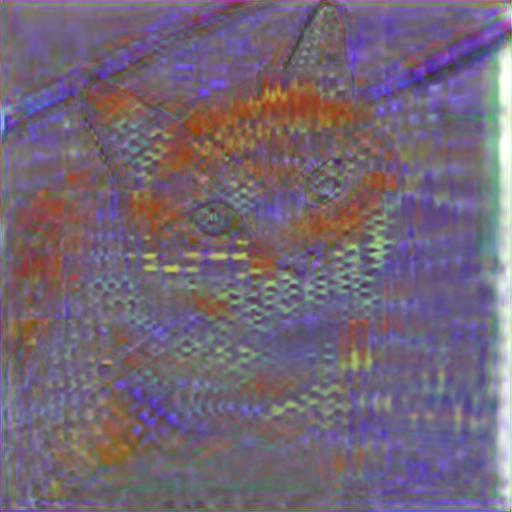} &
\includegraphics[width=0.118\linewidth]{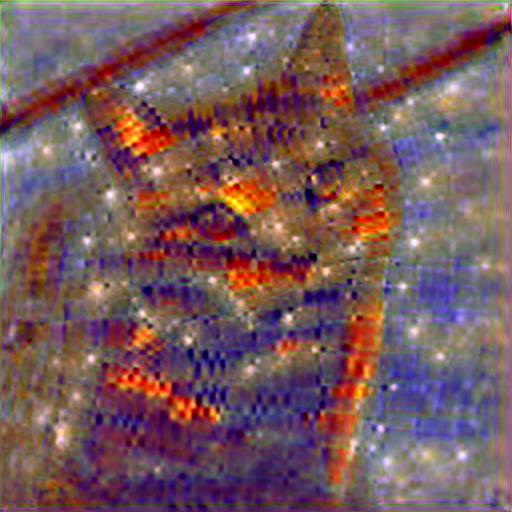} &
\includegraphics[width=0.118\linewidth]{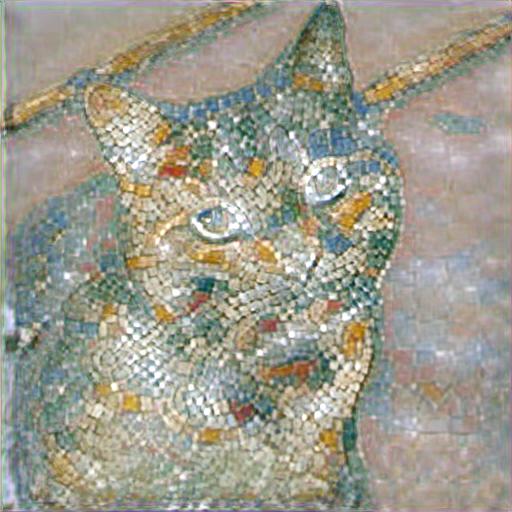}
\\

& & & L-1 & L-2 & L-3 & & L-1 & L-2 & L-3 \\  
\rotatebox[origin=l]{90}{~~Texture} & 
\includegraphics[width=0.118\linewidth]{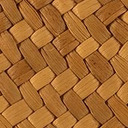} &
\rotatebox[origin=l]{90}{~~~~PNF} & 
\includegraphics[width=0.118\linewidth]{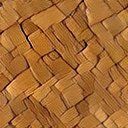} &
\includegraphics[width=0.118\linewidth]{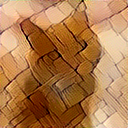} &
\includegraphics[width=0.118\linewidth]{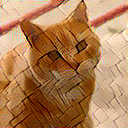} & 
 
\rotatebox[origin=l]{90}{~~~~PNF} & 
\includegraphics[width=0.118\linewidth]{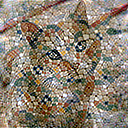} &
\includegraphics[width=0.118\linewidth]{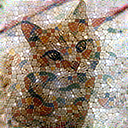} &
\includegraphics[width=0.118\linewidth]{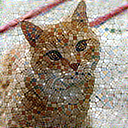} 
\\

& & & \multicolumn{3}{c}{(a)} & & \multicolumn{3}{c}{(b)} \\
\end{tabular}
\vspace{-0.2cm}
\caption{Texture transfer. (a). We optimize specific layers of the neural field. L-1 (layers 1-4), L-2 (layers 2-4), L-3 (layers 3-4). 
The texture contains stationary texture (in high frequency) and structured texture (in low frequency). As opposed to \bacon{},  our method can isolate the stationary texture. (b) We consider the text based texture transfer objectives of CLIPStyler~\cite{kwon2021clipstyler} for the cat content image. For PNF, we consider the text prompt ``Mosaic'' and apply the same layer-based optimization as in (a). For comparison, we apply CLIPStyler of ``Low frequency mosaic'' (B-1),  ``High frequency mosaic'' (B-3) and ``Mosaic'' (B-3). }
\label{fig:texture}
\end{center}
\end{figure*}

Traditional approaches for image manipulation assume the input and output images are represented using a regular grid~\cite{gatys2016image, efros2001image}. Recent work has opted to use neural fields instead, for example allowing fine detailed texturing of 3D meshes~\cite{michel2021text2mesh}. Our formulation allows for an additional layer of control. In particular, due to our subband decompositionality, one can restrict the manipulation to particular subbands. The manipulation can then be driven by optimizing the PNF weights corresponding to those  subbands, using various loss objectives. 

We demonstrate this in the setting of texture transfer. We consider a Fourier PNF with four layers of the following frequency ranges (in Hz): (1) $[0, 8]$, (2) $[4, 16]$, (3) $[12, 32]$ and (4) $[28, 64]$. We consider a \textit{content image} C of $128^2$ resolution. In the first stage, we train a network to fit C as in \cref{sec:expressivity}. In the second stage, we optimize only the parameters of specifics layers. To optimize these parameters, we query the network on a $128^2$ image grid producing image I. We then consider two sets of objectives: (a) Content and style loss objectives as given in \cite{gatys2015neural}. (b), Text-based texture manipulation objectives as given by CLIPStyler~\cite{kwon2021clipstyler}. See further details in the supplementary. 

In \cref{fig:texture}, we illustrate the result of optimizing layers 1 to 4, 2 to 4 or 3 to 4, corresponding to frequency ranges [0, 64], [4, 64]  and [12, 64] respectively. We consider a texture which contains both stationary texture in the high frequency range and structured texture in the low frequency range. As opposed to BACON,  our method can isolate the stationary texture. As can be seen in \cref{fig:texture}(b), for text based manipulation, our method generates texture that is in the correct frequency range. For comparison, we consider CLIPStyler~\cite{kwon2021clipstyler}, with text prompts of "Low frequency mosaic" (B-1),  "High frequency mosaic" (B-3) and "Mosaic" (B-3), resulting in non-realistic texturing. As can be seen, simply specifying the frequency in the text does not result in a satisfactory result.

\subsection{Scale-space Representation}
\label{sec:scalespace}

In many visual computing applications such as volumetric rendering, one is usually required to aggregate information from the neural fields using operations such as Gaussian convolution~\cite{barron2021mip}.
It is useful to model signals as a function of both the spatial coordinates and the scale: $f(\xx, \Sigma)=\mathbb{E}_{\xx\sim\mathcal{N}(\xx, \Sigma)}[g(\xx)]$, where $g$ is the assumed ground truth signal.
Existing works try to approximate the scale-space by using a black-box MLP with intergrated positional encoding, which computes the analytical Gaussian convolved Fourier basis functions~\cite{barron2021mip,Verbin2021RefNeRFSV}.
While such approaches demonstrate success in volumetric rendering applications which requires integrating different scales, they depend on supervision in multiple scales, possibly because their ability to interpolate correctly between scales is hindered by the black-box MLP.
\begin{wrapfigure}{R}{0.5\linewidth}
\vspace{-3em}
\begin{center}
\begin{tabular}{c@{~}c@{}c@{}c@{}c}
&
1$\times$ &
1/2$\times$ &
1/4$\times$ &
1/8 $\times$ 
\\
\rotatebox[origin=l]{90}{BACON} &
\includegraphics[width=0.23\linewidth]{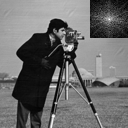} &
\includegraphics[width=0.23\linewidth]{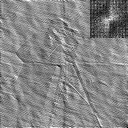} &
\includegraphics[width=0.23\linewidth]{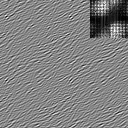} &
\includegraphics[width=0.23\linewidth]{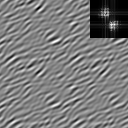} \\
\rotatebox[origin=l]{90}{~~~~~IPE} &
\includegraphics[width=0.23\linewidth]{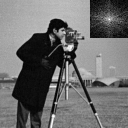} &
\includegraphics[width=0.23\linewidth]{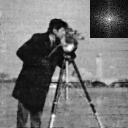} &
\includegraphics[width=0.23\linewidth]{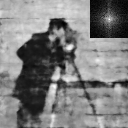} &
\includegraphics[width=0.23\linewidth]{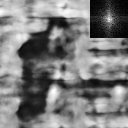} \\
\rotatebox[origin=l]{90}{~~~~PNF} &
\includegraphics[width=0.23\linewidth]{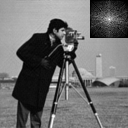} &
\includegraphics[width=0.23\linewidth]{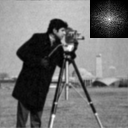} &
\includegraphics[width=0.23\linewidth]{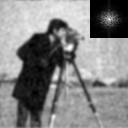} &
\includegraphics[width=0.23\linewidth]{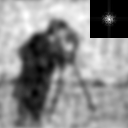} \\ 
\rotatebox[origin=l]{90}{~~~GTR} &
\includegraphics[width=0.23\linewidth]{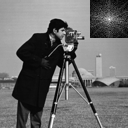} &
\includegraphics[width=0.23\linewidth]{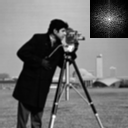} &
\includegraphics[width=0.23\linewidth]{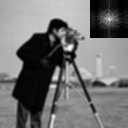} &
\includegraphics[width=0.23\linewidth]{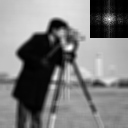}
\end{tabular}
\caption{Scale-space representation. 
Networks are trained on full resolution image (1x) and test on the other reolutions.
GTR is produced by applying Gaussian smoothing on (1x) image. 
}
\vspace{-2em}
\label{fig:gaussian-2d}
\end{center}
\end{wrapfigure}

In this section we want to demonstrate the PNF's ability to better model this scale space with limited supervision.
Suppose the signal of interest can be represented by Fouier bases as $g(\xx) = \sum_n \alpha_n \exp{(\omega_i^T\xx)}$, then we know analytically the Gaussian convolved version should be $f(\xx, \Sigma) = \sum_n \alpha_n\exp(\omega_i^T\Sigma\omega_i)\exp(i\omega_i^T\xx)$.
If we assume that our Fourier PNF can learn the ground truth representation well, then one potential way to achieve this is setting $\gamma(\xx, \Sigma)_n=\exp(-\frac{1}{2}\omega_n^T\Sigma\omega_n)\exp(-i\omega_n^T\xx)$.
We also multiply the output of $F_i$ (\cref{eq:ff-pnf-within-set}) with a correction term addressing the error arises from missing the interference terms in the form of $\exp(-\frac{1}{2}\omega_i^T\Sigma\omega_j)$.
We show in the supplementary how to derive and approximate these missing terms using Fourier PNF.

We test the effectiveness of Fourier PNF to learn the scale-space of a 2D image.
In this application, the network is trained with the image signal in full resolution (finest scale).
At test time, the network is asked to produce image with different scales and compared to the ground truth Gaussian smoothed image.
We compared our method with IPE~\cite{barron2021mip} as well as BACON with IPE as filter function.
The results are shown in \cref{fig:gaussian-2d}.
Our model can represent a signal reasonably well when testing with a lower resolution while other methods degrade more quickly.

\paragraph{Limitations}
Currently, the activation memory of PNF scales linearly in the number of subbands, and so interpretability and decompositiality gained by PNF comes at a ``cost" of a larger memory footprint. 
Further, it is also non trivial to tile higher dimensional space with controllabel subbands.

\section{Conclusion}

We proposed a novel class of neural fields (PNFs) that are compact and easy to optimize and enjoy the interpretability and controllability of signal processing methods. We provided a theoretical framework to analyse the output of a PNF and design a Fourier PNF, whose outputs can be decomposed in a fine-grained manner in the frequency domain with user-specified lower and upper limits on the frequencies. We demonstrated that PNFs matches state-of-the-art performance for signal representation tasks. We then demonstrated the use of PNF's subband decomposition in the settings of texture transfer and 
scale-space representations. 
As future work, the ability to generalize our representation to represent multiple higher-dimensional signals (such as multiple images) can enable applications in recognition and generation, where one can leverage our decomposeable architecture to impose a prior or regularize specific subbands to improve generalization. 

\paragraph{Acknowledgement}
This research was supported in part by the Pioneer Centre for AI, DNRF grant number P1. Guandao’s PhD was supported in part by research gifts from Google, Intel, and Magic Leap. Experiments are supported in part by Google Cloud Platform and GPUs donated by NVIDIA.

\bibliographystyle{plainnat}
\bibliography{ref}

\clearpage
\appendix

\section{Theory}
In this section, we will provide a detail derivation of the theories we used for building and analyzing PNF.
We first study the definition of PNF in \cref{sec:definition}.
Then we prove the properties that PNF is linear sums of basis for different basis in \cref{sec:basis} and \cref{sec:linear-sum-basis}.
In \cref{sec:controllable-sets-of-subbands}, we studies how to organize subset of basis, or subbands, in a controllable way to produce subband-limited PNFs.
Finally, we show several different instantiation of PNFs in \cref{sec:diff-pnfs}.

\subsection{Definition}\label{sec:definition}
\begin{definition}[PNF]
	Let $\bB$ be a basis for the vector space of functions for $\Re^n \to \Re$.
	A \emph{Polynomial neural field of basis $\bB$} is a neural network $f = g_L \circ \dots \circ g_1 \circ \gamma$,
	where $\forall i, g_i$ are finite degree multivariate polynomials, 
	and $\gamma:\Re^n \to\Re^d$ is a $d$-dimensional feature encoding using basis $\bB$: $\gamma(x) = [\gamma_1(x), \dots, \gamma_d(x)]^d, \gamma_i \in \bB, \forall i$. 
\label{def:pnf_def}
\end{definition}

We will show how this definition has included both $\Pi$-Net~\cite{chrysos2020poly}, MFN~\cite{fathony2020multiplicative}, and BACON~\cite{Lindell2021BACONBC}.

\paragraph{$\Pi$-Net.} Here we will set basis $\bB=\{x^n\}_{n\geq 0}$. Then $g_i$ can be set according to different factorization mentioned in Section 3.1 and 3.2. 

\paragraph{MFN.} \cite{fathony2020multiplicative} studied two types of MFNs - Fourier and Gabor MFN. For Fourier MFN that that takes $\Re^d$ as input, we will set basis to be $\bB=\{\sin(\omega^Tx)\}_{\omega\in\Re^d}$. Assume there are $L_{mfn}$ layers of the multiplicative filter networks and each layer has hidden dimension of $h$.
We will set $L=L_{mfn} + 2$ and define $g_i$, $\gamma$ in the following manner:
\begin{align}
\gamma &\in \Re^d \to \Re^{L_{mfn}h}, \\
g_1(\gamma(x)) &= [M_1\gamma(x), \gamma(x)] \in \Re^{(L_{mfn}+1)h}\\
g_i([z, \gamma(x)]) &= [(W_i z + b_i) \odot (M_i \gamma(x)), \gamma(x)],\quad\forall\ 2 \leq i \leq 2 + L_{mfn} \\
g_{L_{mfn} + 2}([z, \gamma(x)]) &= W_{out}x + b_{out},
\end{align}
where $M_i\in\Re^{L_{mfn}h \times h}$ selects the $(i-1)h$ to $ih$ basis by setting each row of $M_i$ to be an one-hot vector and only the $((i-1)h)^{th}$ to $(ih)^{th}$ columes are non-zeros.
Similarly, for Gabor MFN, we use the same definition of $g$, but switch $\gamma$ to sample from $\bB=\{\exp(\gamma\norm{x-\mu}^2)\sin(\omega^Tx)\}_{(\gamma, \mu, \omega \in \Re^d)}$.

\paragraph{BACON.} The way to instantiate BACON will be similar to MFN. Basically each intermediate output layer of BACON is a MFN with $\gamma$ to be sampled from specific subbands.

The definition of PNF is very general such that it not only include prior works but also allow potential design of new architectures with different network topology and differnet basis.
Our design of Fourier PNF will leverage fourier basis with a modified network architecture.
We will introduce several more variants of the PNFs with different architectures and basis choices.

\subsection{Basis function}\label{sec:basis}

\begin{definition}[Span of Basis is Closed Under Multiplication]
We call a basis $\bB$'s span closed under multiplication if :
$\forall b_1, b_2 \in \bB, 
b_1(x)b_2(x) = \sum_{i \in I} a_i b_i(x), |I| < \infty$.
\label{def:basis}
\end{definition}
Note that this is the same requirement as Definition 1 in the appendix of MFN paper~\cite{Fathony2021MultiplicativeFN}. 
We will extend the analysis of MFN in several ways.
First, we will show that several commonly used basis functions satisfies \cref{def:basis}.

\begin{lemma}[Fourier Basis]
Assume the Fourier basis of functions $\Re^d\to\Re$ takes the form of $\bB_{Fourier} = \{ b_{\omega} = \exp(i\omega^Tx)| \omega \in \Re^d\}$. 
Then $\bB_{Fourier}$'s span is closed under multiplication.
\label{lemma:fourier-basis}
\end{lemma} 
\begin{proof} It's enough to show that the multiplication of two Fourier basis function is still a Fourier basis funciton:
$
    \exp(i\omega_1^Tx)\exp(i\omega_2^Tx) = \exp(i(\omega_1 +\omega_2)^Tx)\in \bB_{Fourier}
$.
\end{proof}

\begin{lemma}[RBF]
Assume the Radial basis functions for real-value functions $\Re^d\to\Re$ takes the form of $\bB_{RBF} = \{b_{\gamma,\mu} = \exp(-\frac{1}{2}\gamma\norm{x-\mu}^2)$. 
Then $\bB_{RBF}$'s span is closed under multiplication.
\label{lemma:rbf-basis}
\end{lemma}
\begin{proof}
Similar to Fourier Basis, we will show that the multiplication of two RBF functions is still in $\bB_{RBF}$. This is shown in Equation 24 of Supplementary of MFN~\cite{Fathony2021MultiplicativeFN}:
\begin{align}
    &\quad\exp\paren{-\frac{1}{2}\gamma_1\norm{x-\mu_1}^2}
    \exp\paren{-\frac{1}{2}\gamma_2\norm{x-\mu_2}^2} 
    \\
    &= \exp\paren{\frac{-\gamma_1\gamma_2\norm{\mu_1-\mu_2}}{2(\gamma_1+\gamma_2)}}
    \exp\paren{-\frac{1}{2}(\gamma_1+\gamma_2)\norm{x-\frac{\gamma_1\mu_1+\gamma_2\mu_2}{\gamma_1+\gamma_2}}^2}\\
    &= c(\gamma_1, \gamma_2, \mu_1,\mu_2)\exp\paren{-\frac{1}{2}\gamma'\norm{x-\mu'}^2}, \label{eq:rbf-mult}
\end{align}
where $c(\gamma_1, \gamma_2, \mu_1,\mu_2) = \exp\paren{\frac{-\gamma_1\gamma_2\norm{\mu_1-\mu_2}}{2(\gamma_1+\gamma_2)}}$, $\mu' = \frac{\gamma_1\mu_1+\gamma_2\mu_2}{\gamma_1+\gamma_2}$, and $\gamma' = \gamma_1 + \gamma_2$.
\end{proof}

\begin{lemma}[Gabor Basis]
Assume the Gabor basis of functions $\Re^d\to\Re$ takes the form of $\bB_{Gabor} = \{b_{\gamma,\mu,\omega} = \exp(-\frac{1}{2}\gamma\norm{x-\mu}^2)\exp(i\omega^Tx)| \omega \in \Re^d, \mu \in \Re^d, \gamma\geq 0\}$. 
Then $\bB_{Gabor}$'s span is closed under multiplication.
\label{lemma:gabor-basis}
\end{lemma}
\begin{proof}
Using \cref{eq:rbf-mult}, we can compute the the multiplication of two Gabor basis functions:
\begin{align}
b_{\gamma_1,\mu_1,\omega_1}b_{\gamma_2,\mu_2,\omega_2} = c(\gamma_1, \gamma_2, \mu_1,\mu_2)\exp\paren{-\frac{1}{2}\gamma'\norm{x-\mu'}^2}\exp(i\omega'^Tx) \propto b_{\gamma',\mu',\omega'} \in \bB_{Gabor},
\end{align}
where $\omega'=\omega_1+\omega_2$. 
The output of the multiplication is still a Gabor.
\end{proof}

The following Lemma will show that \cref{def:basis} can be extended to analyzing functions from different domain. 
We will show that for complex-value function that maps from a sphere, there is a basis function that satisfies \cref{def:basis}.
\begin{lemma}[Spherical Harmonics]
We will consider the basis function for real functions that takes spherical coordinate (i.e. $S^2\to \Ce$ where $S^2=\{(\theta, \phi) | 0\leq \theta \leq \pi, 0\leq \phi \leq 2\pi\}$.
Moreover, we will consider Laplace's spherical harmonics as basis:
\begin{align}
    \bB_{SH}=\{Y_l^m(\theta,\phi)=e^{im\phi}P_l^m(\cos(\theta)) | 0 \leq l, l \in \Ze, -l \leq m \leq l, m \in \Ze \},\label{eq:sh-basis-def}
\end{align}
where $P_l^m:[-1,1]\to\Re$ is an associated Legendre polynomial.
$\bB_{SH}$ satisfies \cref{def:basis}.
\label{lemma:sh-basis}
\end{lemma}
\begin{proof}
From the multiplication rule of Spherical Harmonics~\cite{AngularMomentum}, we have:
\begin{align}
    Y_{l_1}^{m_1}(\theta, \phi)Y_{l_2}^{m_2}(\theta,\phi) \propto 
    \sum_{l=0}^\infty \sum_{m=-c}^c (-1)^m\sqrt{2l+1}
    \begin{pmatrix}
    l_1\ \ l_2\ \ l \\
    m_1\ \ m_2\ \ -m
    \end{pmatrix}
    \begin{pmatrix}
    l_1\ \ l_2\ \ l \\
    0 \ \ 0\ \ 0
    \end{pmatrix}
    Y_l^m(\theta,\phi),\label{eq:sh-contract}
\end{align}
where $\begin{pmatrix} j_1\ \ j_2\ \ j_3\\ m_1\ \ m_2\ \ m_3 \end{pmatrix}$ denotes the $3j$-syombols.
Now we need to show that the infinite sum contains only finite number of non-zero terms.
By the selection rules of the $3j$-symbols~\cite{MathForPhysics}, we know that 
$\begin{pmatrix} j_1\ \ j_2\ \ j_3\\ m_1\ \ m_2\ \ m_3 \end{pmatrix}$ is zero if any of the following rules is not satisfies: 1) $|j_1 - j_2| \leq j_3 \leq j_1 + j_2$; and 2) $m_1 + m_2 + m_3 = 0$.
This implies that for all terms $l \geq l_1+l_2$, and $m\neq -(m_1+m_2)$, the term $\begin{pmatrix}l_1\ \ l_2\ \ l \\ m_1\ \ m_2\ \ -m\end{pmatrix} = 0$.
With this said, the \cref{eq:sh-contract} can be written as a finite sum:
\begin{align}
    Y_{l_1}^{m_1}(\theta, \phi)Y_{l_2}^{m_2}(\theta,\phi) \propto 
    \sum_{l=|l_1-l_2|}^{l_1+l_2}
    \begin{pmatrix}
    l_1\ \ l_2\ \ l \\
    m_1\ \ m_2\ \ -(m_1+m_2)
    \end{pmatrix}
    \begin{pmatrix}
    l_1\ \ l_2\ \ l \\
    0 \ \ 0\ \ 0
    \end{pmatrix}
    Y_l^m(\theta,\phi).\label{eq:sh-contract-finite}
\end{align}
\end{proof}

There are several rules  we can used to generate more basis that satisfies the properties.
Here we will show one of them:
\begin{theorem}[Basis-multiplication.]
If $\bB_1$ and $\bB_2$ be the basis of $\Re^d\to\Re$ and satisfy \cref{def:basis}, then $\bB_{mult} = \{b_1b_2 | b_1 \in \bB_1, b_2 \in \bB_2\}$ also satisfy \cref{def:basis}.
\label{thm:basis-composition}
\end{theorem}

\begin{proof}
Let $b_1  b_2$ and $b_3  b_4$ from $\bB_{mult}$.
Assume $b_i = \sum_n a_{i,n} b_{k_{i,n}}$, where $i\in\{1,2,3,4\}$, $k_{i,n}$ is an index into $\bB_1$ if $i \in \{1, 2\}$ and $\bB_2$ if $i \in \{3, 4\}$, and $a_{i,n}$ are coefficients. Then we have
\begin{align}
    (b_1 b_2)(b_3 b_4) &= 
    \paren{\sum_n a_{1,n} b_{k_{1,n}}} 
    \paren{\sum_n a_{2,n} b_{k_{2,n}}} 
    \paren{\sum_n a_{3,n} b_{k_{3,n}}} 
    \paren{\sum_n a_{4,n} b_{k_{4,n}}} \\
    &= 
    \paren{\sum_{n,n'} a_{1,n} a_{3,n'} b_{k_{1,n}} b_{k_{3,n'}}} 
    \paren{\sum_{n,n'} a_{2,n} a_{4,n'} b_{k_{2,n}} b_{k_{4,n'}}}.\label{eq:multi-last-equation}
\end{align}
Since both $b_{k_{1,n}}$ and $b_{k_{3,n'}}$ are in $\bB_1$ and $\bB_1$ satisfies \cref{def:basis}, so we can assume the following (and similarly logics can be applied to $b_{k_{2,n}}$ and $b_{k_{4,n'}}$ for $\bB_2$):
\begin{align}
    b_{k_{1,n}}b_{k_{3,n'}} &= \sum_{m} c_{1,n,n',m} b_{l_{1,n,n',m}},\quad b_{l_{1,n,n',m}}\in\bB_1 \\
    b_{k_{2,n}}b_{k_{4,n'}} &= \sum_{m} d_{n,n',m} b_{l_{2,n,n',m}},\quad b_{l_{2,n,n',m}}\in\bB_2.
\end{align}
Plugging abovementioned equations into \cref{eq:multi-last-equation} we get
\begin{align}
    \sum_{n,n'} a_{1,n} a_{3,n'} b_{k_{1,n}} b_{k_{3,n'}} 
    &= \sum_{n,n'} a_{1,n} a_{3,n'} \paren{\sum_{m} c_{1,n,n',m} b_{l_{1,n,n',m}}} \\
    &= \sum_{n,n',m} a_{1,n} a_{3,n'} c_{1,n,n',m} b_{l_{1,n,n',m}} = \sum_{n,n',m} \tilde{c}_{1,n,n',m} b_{l_{1,n,n',m}},
\end{align}
for $\tilde{c}_{1,n,n',m} = a_{1,n} a_{3,n'} c_{1,n,n',m}$.
Similarly, we have 
\begin{align}
    \sum_{n,n'} a_{2,n} a_{4,n'} b_{k_{2,n}} b_{k_{4,n'}} 
    &= \sum_{n,n'} a_{2,n} a_{4,n'} \paren{\sum_{m} c_{2,n,n',m} b_{l_{2,n,n',m}}} \\
    &= \sum_{n,n',m} a_{2,n} a_{4,n'} c_{2,n,n',m} b_{l_{2,n,n',m}} = \sum_{n,n',m} \tilde{c}_{2,n,n',m} b_{l_{2,n,n',m}},
\end{align}
for $\tilde{c}_{2,n,n',m} = a_{2,n} a_{4,n'} c_{2,n,n',m}$.
Finally we can put these two together:
\begin{align}
    (b_1b_2)(b_3b_4) &=\paren{\sum_{n,n',m} \tilde{c}_{1,n,n',m} b_{l_{1,n,n',m}}}\paren{\sum_{n,n',m} \tilde{c}_{2,n,n',m} b_{l_{2,n,n',m}}} \\
    &= \sum_{n,n',m,n'',n''',m'}\tilde{c}_{1,n,n',m}\tilde{c}_{2,n'',n''',m'} b_{l_{1,n,n',m}} b_{l_{2,n'',n''',m'}},
\end{align}
which is a linear combination of basis $\bB_{mult}$ since $ b_{l_{1,n,n',m}} \in \bB_1$ and $b_{l_{1,n'',n''',m'}}\in\bB_2$.
\end{proof}

\subsection{PNF as Linear Sum of Basis}\label{sec:linear-sum-basis}
In this section, we will show that PNF can be expressed as linear sum of basis if the basis function satisfies \cref{def:basis}.
\begin{lemma}[Power-product of Basis]
If $\bB$ satisfies \cref{def:basis}, then power-products of the form $\prod_{n=1}^{N} b_n^{\alpha_n}$, $0 \leq \alpha_n < \infty$ is a linear sum of the basis $\bB$.
\label{lemma:linear-sum-product-basis}
\end{lemma} 
\begin{proof}
We will show by induction on the degree of the power-product: $d = \sum_{n}\alpha_n$.

\textit{Base cases.} If $d=0$, then $\prod_{n=1}^{N} b_n^{\alpha_n} = 1 \in \bB$ since $\bB$ is a basis function of $\Re^d \to \Re$.

\textit{Inductive case.} The inductive hypothesis is: assume that for some degree $d$ such that $d \geq 0$, all finite power-product of degree $d$ are linear sums of basis $\bB$. We want to show that all finite power-product of degree $d+1$ is also linear sums of basis $\bB$. Condider a power-product of degree $d+1$ in the form of $\prod_{n=1}^{N} b_n^{\alpha_n}$ and assume without lost of generality that $\alpha_1 \geq 1$. Then we have
\begin{align}
    \prod_{n=1}^{N} b_n^{\alpha_n} &= b_1 \paren{b_1^{(\alpha_n - 1)}\prod_{n=2}^{N} b_n^{\alpha_n}} = b_1 p,
\end{align}
where $p = b_1^{(\alpha_n - 1)}\prod_{n=2}^{N} b_n^{\alpha_n}$.
It's easy to see that the degree of $p$ is $d$, so $p$ can be written as a linear sum of $\bB$: $\sum_m a_m b_m$. With that we have:
\begin{align}
    \prod_{n=1}^{N} b_n^{\alpha_n} = b_1 \sum_m a_m b_m  =  \sum_m a_m b_1b_m, \label{eq:lamma-last}
\end{align}
where $b_i\in\bB$.
Since $\bB$ satisfies \cref{def:basis}, so $b_1b_m$ can be written as a linear sum of basis $\bB$: $\sum_{j} c_j b_j$. 
Putting this in \cref{eq:lamma-last} gives a linear sum of basis $\bB$.
\end{proof}

\begin{theorem}[]
Let $F$ be a PNF with basis $\bB$.
$\forall b_1, b_2 \in \bB, 
b_1(x)b_2(x) = \sum_{i \in I} a_i b_i(x), |I| < \infty
$, then the output of $F$ is a finite linear sum of the basis functions from $\bB$.
\label{thm:interp}
\end{theorem}
\begin{proof}
Let $F=g_L \circ \dots \circ g_1 \circ \gamma$, with $g_i$ being a multivariate polynomial.
Let $z_0 = \gamma$.
Let $z_i = g_i \circ \dots \circ g_1 \circ \gamma$ for $i \geq 1$.
Let $z_i[j]$ be the $j^{th}$ dimensional value of $z_i$.
We will show by induction that for all $i$, $z_i[j]$ is a linear sum of basis functions from $\bB$ for all $j$.

\textit{Base case : $i=0$.} $z_i[j] = \gamma_j(x) \in \bB$ by definition of $\gamma$ in \cref{def:pnf_def}. 

\textit{Inductive case.} The inductive hypothesis is: for $k\geq 1$, if $z_k[j]$ is linear sum of $\bB$ for all $j$, then $z_{k+1}[l]$ is linear sums of $\bB$ for all $l$. 

By definition of $z$, we know that $z_{k+1} = g_k(z_k)$, where $g_k$ is a multivariate polynomial of finite degree $d$. With that said, we can assume $z_{k+1}[l]\in \Re$ is a linear sum of power-product terms in the form of 
$\prod_{j} z_k[i]^{\alpha_{lj}}$, where $\alpha_{lj} \geq 0$ and $\sum_l \alpha_l \leq d$.
It's sufficient to show that each of this term is linear sum of the basis function $\bB$. 

By the inductive hypothesis, we can assume that $z_k[j]=\sum_n \beta_{jn}b_{jn}$, with $b_{jn} \in \bB$. Then we have:
\begin{align}
    \prod_{j} z_k[i]^{\alpha_{lj}} 
    = \prod_{j} \paren{\sum_n \beta_{jn}b_{jn}}^{\alpha_{lj}} 
    = \sum_{m} a_{m} \prod_n b_n^{\hat{\alpha}_{nm}} \label{eq:linear-sum-lastline}, 
\end{align}
for some $a_m\in\Re$, and $\hat{\alpha}_{nm} \geq 0$ and $\sum_m \hat{\alpha}_{nm} \leq d$. 
By \cref{lemma:linear-sum-product-basis}, the terms $\prod_n b_n^{\hat{\alpha}_{nm}}$ are linear sums of $\bB$. 
Then \cref{eq:linear-sum-lastline} is linear sums of linear sums of $\bB$, which will be linear sums of $\bB$.
\end{proof}

\subsection{Controllable Sets of Subbands}\label{sec:controllable-sets-of-subbands}

In this section, we will develop theories to build controllable sets of subbands and use that to design PNFs.
In this section, we will use definition that a subband is a subset of the basis function.
\begin{definition}[Subband]
A subband $S$ is a subset of $\bB$.
\end{definition}

\begin{definition}[Subband limited PNF]
A PNF $F$ of basis $\bB$ is limited by subband $S \subset \bB$ if each dimension of the output (i.e. $F_i$) is in the span of  $S$. 
\label{def:subband-limited-def}
\end{definition}

One naive way to construct a subband limited PNF is to restrict $\gamma$ to take only basis functions from the subband: $\gamma_i \in S$, and then restricted no multiplication in the layers $g_i$ (since multiplication is the only operation that can change the composition of the basis functions used in the linear sum).

But this will simply create a very shallow network as the composition of linear layers amounts to only one linear layer.
As a result, we will need a very wide network in order to achieve expressivity. 
This means the number of basis function we used will grow linear with the number of network parameters.
When it requires exponential number of basis functions to approximate a signal well, then we are not capable of achieving it compactly, throwing away a key virtual of neural fields.

In order to achieve compactness, we want to be able to generate a lots of basis, all of with within the subband of interest, without instantiating a lot of network parameters.
One way to achieve this is to use function composition with non-linearity.
In our case, since the functions are restricted to be polynomials, then the only non-linearity we can use is multiplication.
But multiplication can potentially create basis functions outside of the subband as \cref{def:basis} does not restrict the linear sum to be within certain subband.
As a result, we need to study how multiplication transform the subband.
More specifically, we will first define a set of subbands, within which multiplication in function space can translate in someway to an operation of subbands:
\begin{definition}[PNF-controllable Set of Subbands]
$\sS = \{S_\theta | S_\theta \subset \bB\}_\theta$ is a PNF-Controllable Set of Subbands for basis $\bB$ if (1) $S_{\theta_1} \cup S_{\theta_2} \in \sS$ and (2) there exists a binary function $\otimes: \sS \times \sS \to \sS$ such that 
if $b_1\in S_{\theta_1}, b_2\in S_{\theta_2} \implies b_1 b_2 = \sum_n a_n b_n, b_n \in S_{\theta_1} \otimes S_{\theta_2}$ for some coefficients $a_n\in\Re$ (i.e. $b_1b_2$ is in the span of $S_{\theta_1} \otimes S_{\theta_2}$).
\label{def:controllable-set-of-subbands}
\end{definition}

\begin{theorem}[]
Let $\sS$ be a PNF-controllable set of subbands of basis $\bB$ with its corresponding binary function $\otimes$.
Suppose  $F$ and $G$ are polynomial neural fields of basis $\bB$ that maps $\Re^n$ to $\Re^{m_1}$ and $\Re^{m_2}$ respectively.
Furthermore, suppose $F$ and $G$ are subband limited by $S_1 \in \sS$ and $S_2\in \sS$.
Then we have the following:
\begin{enumerate}
    \item $A(x) = \ww_1^TF(x)+\ww_2G(x)$ is a PNF of $\bB$ limited by subband $S_1 \cup S_2$ with $\ww_1 \in \Re^{m_1}$ and $\ww_2 \in \Re^{m_2}$\footnote{In the main text, we show a $o$-dimensional version: $W_1F(x) + W_2G(x)$, where $W_1\in \Re^{m_1\times o}$ and $W_2\in \Re^{m_2\times o}$. The $o$-dimensional version can be seen as an easy extension of this single dimensional version as we can view $W_1=[\ww_{11}^T; \dots; \ww_{1o}^T]$ where $\ww_{1i}\in \Re^{m_1}$ and similarly for $W_2$.}; and
    
    \item $M(x) = F(x)^TWG(x)$ is a PNF of $\bB$ limited by subband $S_1 \otimes S_2$ with $W \in \Re^{m_1 \times m_2}$.
\end{enumerate}
\label{thm:controllability}
\end{theorem}
\begin{proof}
By \cref{def:subband-limited-def}, we can assume that 
\begin{align}
	\forall 1\leq i \leq m_1, F_i(x) = \sum_{j} a_{f,j} b_{f,j}(x),\quad b_{f,j} \in S_1 \\ 
	\forall 1\leq i \leq m_2, G_i(x) = \sum_{j} a_{g,j} b_{g,j}(x),\quad b_{g,j} \in S_2.
\end{align}
Then we can compute $A(x)$:
\begin{align}
	A(x) 
	&= \sum_{i=1}^{m_1} \ww_1[i]F_i(x) + \sum_{i=1}^{m_2} \ww_2[i]G_i(x) \\
	&= \sum_{i=1}^{m_1} \ww_1[i]\paren{\sum_{j} a_{f,j} b_{f,j}(x)} + \sum_{i=1}^{m_2} \ww_2[i]\paren{\sum_{j} a_{g,j} b_{g,j}(x)} \\
	&= 
	\underbrace{\paren{ \sum_{i=1}^{m_1}\sum_{j} (\ww_1[i] a_{f,j}) b_{f,j}(x)} }_{\text{Subband limited by }S_1}
		+ 
	\underbrace{\paren{\sum_{i=1}^{m_2}\sum_{j} (\ww_2[i] a_{g,j}) b_{g,j}(x)}}_{\text{Subband limited by }S_2}.\label{eq:add-control-last-line}
\end{align}
It's easy to see that \cref{eq:add-control-last-line} is subband limited by $S_1\cup S_2$.

Similarly, we can apply the computation to $M(x)$:
\begin{align}
	M(x) 
	&= \sum_{i,j\geq 1}^{i,j\leq m_1,m_2} W[i,j]F_i(x)G_j(x) \\
	&= \sum_{i,j\geq 1}^{i,j\leq m_1,m_2} W[i,j]
		\paren{\sum_{k} a_{f,k} b_{f,k}(x)}
		\paren{\sum_{k} a_{g,k} b_{g,k}(x)} \\
	&= \sum_{i,j\geq 1}^{i,j\leq m_1,m_2} W[i,j]
		\paren{\sum_{k,k'}a_{f,k}a_{g,k'}b_{f,k}(x)b_{g,k'}(x)} \\
	&= \sum_{i,j\geq 1,k,k'}^{i,j\leq m_1,m_2} 
		(W[i,j]a_{f,k}a_{g,k'}) b_{f,k}(x)b_{g,k'}(x). \label{eq:mult-control-last-line}
\end{align}
By \cref{def:controllable-set-of-subbands}, we can assume
\begin{align}
b_{f,k}(x)b_{g,k'}(x) = \sum_l c_{k,k',l}b_{k,k',l}(x),\quad b_{k,k',l}(x) \in S_1 \otimes S_2.
\end{align} 
Putting it into \cref{eq:mult-control-last-line}, we have:
\begin{align}
	M(x) 
	&= \sum_{i,j\geq 1,k,k'}^{i,j\leq m_1,m_2} 
	(W[i,j]a_{f,k}a_{g,k'}) \sum_l c_{k,k',l}b_{k,k',l}(x) \\
	&= \sum_{i,j\geq 1,k,k',l}^{i,j\leq m_1,m_2} 
	(W[i,j]a_{f,k}a_{g,k'}c_{k,k',l}) b_{k,k',l}(x),
\end{align}
which is subband limited by $S_1 \otimes S_2$ since $b_{k,k',l} \in S_1 \otimes S_2$.
\end{proof}
With this theorem, we are able to leverage the PNF-controllable sets of subbands to compose PNFs with multiplication and additions in a band-limited ways.
For example, if we want to construct a PNF that's band-limited by $S$, we can do it in the following steps:
\begin{enumerate}
    \item Identify a PNF-controllable set of subbands that contains $S$, let that be $\sS$ with the associated operation to be $\otimes$.
    \item Factorize $S$ into a series of subband by the associated operation : $S=S_1\otimes S_2\otimes \dots \otimes S_n$.
    \item For each $S_i$, we can creates a shallow PNF with $F_i = g^{(i)}\circ \gamma^{(i)}$, where $\gamma^{(i)}_j \in S_i$ for all $j$ and $g^{(i)}(z) = Wz$ is a linear layer (without bias).
    \item Composed these layers together using rule-2 of \cref{thm:controllability}.
\end{enumerate}
One can see that the abovementioned way to construct PNF has the ability to create exponential number of basis functions, all of which within $S$.
Following is the intuitive reason why that's the case.
Suppose each neuron $F_i$ is a linear sum of $d$ different basis functions (i.e. the set of basis chosen for $F_i$ is different from those of $F_j$ if $i\neq j$).
Furthermore, suppose that \cref{def:basis} creates different sets of basis in the right-hand-side of the multiplication.
then every time we apply rule-2 of \cref{thm:controllability} on $d$-dimensional inputs (i.e. assume that we have $d$-numer of $W$ matrix in rule-2 to create an output of $d$-dimensional everytime), we creates $d$-number of different basis for every-single existing cases.
As a result, if we apply rule-2 $L$-times, then we will have $d^L$ number of different basis.
At the same time, the number of parameters we used is $d^2$ for each $F_i$ and $d^3$ for each multiplication.
This means that the number of parameter is $O(Ld^3)$ while the number of basis functions we create in the final linear sum is $O(d^L)$.
As a result, such construction can potentially lead to a compact, expressive, and subband-limited PNF.

Note that what's described above is merely an intuitive argument, since many of the conditions might be difficult to hold strictly (e.g. different basis are created for multiplication on the right-hand-sides of \cref{def:basis}).
Empirically, we found that while these conditions are only loosely held true, such construction is still capable of creating a large number of basis functions.
This observation aligns with the analysis of MFN-like network can a large number of basis functions, presented in MFN~\cite{Fathony2021MultiplicativeFN} and BACON~\cite{Lindell2021BACONBC}.
In the rest of the subsection, we will show some construction of PNF-controllable set of subbands with various basis functions.

\subsubsection{Fourier Basis}
In this section, we will consider Fourier basis of $\Re^d\to \Re$ parameterized by $\omega\in\Re^d$: $\bB_{Fourier}=\{\exp(i\omega^Tx)\}_{\omega}$.
A commonly used subband definition (as shown in Equation (1) of the main text) is following:
\begin{definition}[Fourier Subband]
The Fourier subband for Fourier basis of $\Re^d\to\Re$ functions can be defined with a lower band limit $\alpha \in \Re^+$, an upper band limit $\beta\geq alpha$, an orientation $\dd\in\Re^d$ and an angelar width $\gamma\in\Re^+$:
\begin{align}
R_F(\alpha, \beta, \dd, \gamma,p) =
    \left\{
    \oomega |
    \alpha \leq \norm{\oomega}_p \leq \beta, \norm{\dd}=1,
    \oomega^T\dd \geq \cos(\gamma)\norm{\oomega}_2
    \right\}
    \label{def:ff-subband-def}.
\end{align}
\end{definition}

With these subbands, we can find PNF-controllable set of subbands in the following format:
\begin{theorem}[Fourier PNF-Controllable Set of Subbands with L2 norm]
For $|\gamma| < \frac{\pi}{4}$, define
$\sS_{Fourier-L2}(\dd,\gamma)=\{R_F(\alpha, \beta, \dd, \gamma, 2) | \forall  0 \leq \alpha \leq \beta\}$.
Each subband in $\sS$ can be parameterized with tuple $(\alpha, \beta)$.
$\sS_{Fourier-L2}(\dd, \gamma)$ is a PNF-Controllable Set of Subbands with the following definition of binary relation $\otimes_{FF}$:
\begin{align}
\label{eq:otimes-ff}
    (\alpha_1, \beta_1) \otimes_{FF}(\alpha_2, \beta_2) 
    = (\sqrt{\cos(2\gamma)}(\alpha_1 + \alpha_2), \beta_1 + \beta_2).
\end{align}
\label{thm:ff-control}
\end{theorem}
\begin{proof}
It's sufficient to show that $\omega_1 \in R(\alpha_1, \beta_1, \dd, \gamma, 2)$, and 
$\omega_2 \in R(\alpha_2, \beta_2, \dd, \gamma, 2)$, then 
$\omega_1 + \omega_2 \in R(\sqrt{\cos(2\gamma)}(\alpha_1+\alpha_2), \beta_1 + \beta_2,\dd, \gamma,2)$.
Let's start with the bound in the norm:
\begin{align}
    \norm{\omega_1+\omega_2} 
    &=\sqrt{\norm{\omega_1} + \norm{\omega_2} + 2\omega_1^T\omega_2} \\
    &=\sqrt{\norm{\omega_1}^2 + \norm{\omega_2}^2 + 2\norm{\omega_1}\norm{\omega_2}\cos(\theta)},
\end{align}
where $\theta$ is the radius of the angle between $\omega_1$ and $\omega_2$.
Since $\cos(\theta)\leq 1$, it's easy to show that:
\begin{align}
    \norm{\omega_1+\omega_2} 
    &=\sqrt{\norm{\omega_1}^2 + \norm{\omega_2}^2 + 2\norm{\omega_1}\norm{\omega_2}\cos(\theta)} \\
    &\leq \sqrt{\beta_1^2 + \beta_2^2 + 2\beta_1\beta_2} = \beta_1 + \beta_2.
\end{align}
Now we need to show that $\norm{\omega_1+\omega_2} \geq \alpha_1 + \alpha_2$.
By \cref{def:ff-subband-def}, we know that $\omega_1$'s angle with $\dd$ is at most $\gamma$. 
Similarly, $\omega_2$'s angle with $\dd$ is also at most $\gamma$.
With that said, the largest angle between $\omega_1$ and $\omega_2$ should be less than $2\gamma$.
As a result, $\cos(\theta) \geq \cos(2\gamma)$.
Note that $|\gamma| \leq \frac{\pi}{4}$, so $\cos(\theta) \geq \cos(2\gamma) \geq \cos(\frac{\pi}{2}) = 0$.
With that, we can show:
\begin{align}
    \norm{\omega_1+\omega_2} 
    &=\sqrt{\norm{\omega_1}^2 + \norm{\omega_2}^2 + 2\norm{\omega_1}\norm{\omega_2}\cos(\theta)} \\
    &\geq \sqrt{\alpha_1^2 + \alpha_2^2 + 2\norm{\omega_1}\norm{\omega_2}\cos(\theta)} \\
    &\geq \sqrt{\alpha_1^2 + \alpha_2^2 + 2\alpha_1\alpha_2\cos(2\gamma)} \quad \text{(Possible since } \cos(\theta) \geq \cos(2\gamma)\geq 0\text{)}\\
    &= \sqrt{\cos(2\gamma)(\alpha_1^2 + \alpha_2^2 + 2\alpha_1\alpha_2) + (1-\cos(2\gamma))(\alpha_1^2 + \alpha_2^2)} \\
    &\geq \sqrt{\cos(2\gamma)(\alpha_1^2 + \alpha_2^2 + 2\alpha_1\alpha_2)}\quad\text{Since } \cos(2\gamma)\leq 1 \text{ and } \alpha_1,\alpha_2 \geq 0 \\
    &= \sqrt{\cos(2\gamma)}(\alpha_1 + \alpha_2).
\end{align}
Finally, we want to show that $(\omega_1+\omega_2)^T\dd \geq \cos(\gamma)\norm{\omega_1+\omega}$:
\begin{align}
    (\omega_1+\omega_2)^T\dd 
    &=\omega_1^T\dd + \omega_2^T\dd \\
    &\geq \cos(\gamma)\norm{\omega_1} + \cos(\gamma)\norm{\omega_2} \\
    &= \cos(\gamma)(\norm{\omega_1} + \norm{\omega_2}) \\
    &\geq \cos(\gamma)\norm{\omega_1 + \omega_2}\quad \text{by Triangular inequiality of L2-norm}.
\end{align}
\end{proof}

Intuitively, the smaller the angel $\gamma$ is, the tighter we are able to guarantee the lower-bound compared to $\alpha_1 + \alpha_2$.
With this said, if we want to build a subband with lower-bound $l$ and angle $\gamma$, one way to achieve it is to set $\alpha_1 = 0$ and $\alpha_2 = \frac{l}{\sqrt{\cos(2\gamma)}}$.
But we found empirically that the network perform better when the lower-band-limit is overlap with the upper-band-limit.
We will discuss this more in \cref{sec:implementationdetails}.

\subsubsection{Fourier L1}
As mentioned in Section 3.3.2 of the main paper, when working in image domain, a preferable way to define the PNF-controllable set of subbands is through L-$\infty$ norm.
This is because the L-$\infty$ norm creates subband that can tile the corner of the frequency domain created by the image without going over the Nyquist rate. 
But to achieve it, we need to restrict each subband within a total-vertical or total-horizontal region.
\begin{wrapfigure}{R}{0.25\linewidth}
    \vspace{-5em}
    \centering
    \includegraphics[width=\linewidth,height=\linewidth]{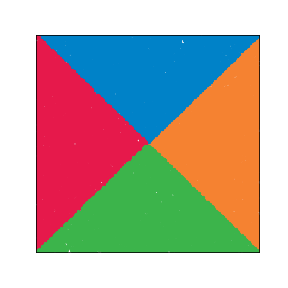}
    \vspace{-1em}
    \caption{Visualization of the four L-$\infty$ consistent regions in $\Re^2$, one color for each.}
    \label{fig:r-infty-region}
\end{wrapfigure}
Intuitively, those regions contains vectors whose L-$\infty$ norm is taking absolute value of the same dimension:
\begin{definition}[L-$\infty$ dimension consistent region]
We call a region $R_\infty(n,s) \subset \Re^d$ has consistent L-$\infty$ dimension of $n$ and sign $s\in \{+1, -1\}$if:
\begin{enumerate}
    \item $\forall \dd\in R, \norm{\dd}_\infty = |\dd[n]|,$ where $\dd[n]$ denotes the $n$th value of vector $\dd$; and
    \item $\forall \dd_1, \dd_2\in R$, $\operatorname{sign}(\dd_1[n]) = \operatorname{sign}(\dd_2[n]) = s$.
\end{enumerate}
\label{def:l-inf-consistent}
\end{definition}
Please refer to \cref{fig:r-infty-region} for the four L-$\infty$ dimension consistent region in $\Re^2$.
In general, for $\Re^d$, there are $2d$ number of such regions.

It's easy to see the following property of $R_\infty(n)$:
\begin{lemma}
If $\dd \in R_\infty(n, s)$, then $a\dd \in R_\infty(n, s)$ for $a\in\Re$ and $a > 0$.
\label{lemma:l-infty-region-scale-invar}
\end{lemma}
\begin{proof}
\begin{align}
    \argmax_i |(a\dd)[i]| \argmax_i |a\dd[i]| = \argmax_i |a||\dd[i]| = \argmax_i |\dd[i]| = n.
\end{align}
Since $a>0$, $\operatorname{sign}(a\dd[n]) = \operatorname{sign}(\dd[n])$.
\end{proof}

With this, we are able to prove a similar version of \cref{thm:ff-control} for L-$\infty$ norm but restricted each set of subband to be within only a region with consistent L-$\infty$ dimension:
\begin{theorem}[Fourier PNF-Controllable Set of Subbands with L-$\infty$ norm]
Let $R_\infty(n, s)$ be a region of consistent L-$\infty$ dimension of $n$ defined for $\Re^d$. 
Define the set of subbands as following:
\begin{align}
\sS_{FF-L\infty}(\dd,\gamma,n)=\{R_F(\alpha, \beta, \dd, \gamma, \infty) | \forall  0 \leq \alpha \leq \beta, R_F(\alpha, \beta, \dd, \gamma, \infty) \subset R_\infty(n,s)\}.
\end{align}
Each subband in $\sS_{FF-L\infty}$ can be parameterized with tuple $(\alpha, \beta)$.
$\sS_{FF-L\infty}(\dd, \gamma)$ is a PNF-Controllable Set of Subbands with the following definition of binary relation $\otimes_{FF}$:
\begin{align}
    (\alpha_1, \beta_1) \otimes_{FF\infty}(\alpha_2, \beta_2) 
    = (\alpha_1 + \alpha_2, \beta_1 + \beta_2).
\end{align}
\label{thm:ff-control-linf}
\end{theorem}
\begin{proof}
Similar to the proof in \cref{thm:ff-control}, we will show that if $\omega_1\in R_F(\alpha_1, \beta_1, \dd, \gamma, \infty)=R_1\in \sS_{FF-\infty}$ and 
$\omega_2\in R_F(\alpha_2, \beta_2, \dd, \gamma, \infty)=R_2\in \sS_{FF-\infty}$,
then we will have 
$\omega_1+\omega_2\in R_F(\alpha_1+\alpha_2, \beta_1+\beta_2, \dd, \gamma, \infty)\in \sS_{FF-\infty}$.

First, we show that $\omega_1+\omega_2$ is still within $R_\infty(n)$.
As shown in the proof for \cref{thm:ff-control}, we have $(\omega_1+\omega_2)^T\dd \geq \cos(\gamma)\norm{\omega_1 + \omega_2}$ .
This means that $\omega_1+\omega_2 \in R_\infty(n)$ by \cref{lemma:l-infty-region-scale-invar}.

Since $R_1$, $R_2$, and $\omega_1+\omega_2$ are both in $\sS_{FF-\infty}$, we will use this property to derive the upper bound:
\begin{align}
    \norm{\omega_1+\omega_2}_\infty = |(\omega_1+\omega_2)[n]| \leq |\omega_1[n]| + |\omega_2[n]| = \norm{\omega_1}_\infty + \norm{\omega_2}_\infty =  \beta_1 + \beta_2.
\end{align}
If $s>0$, then we know $\omega_1[n], \omega_2[n]>0$ and $(\omega_1+\omega_2)[n] > 0$.
As a result, $\norm{\omega_{1,2}}_\infty = \omega_{1,2}[n] \geq \alpha_{1,2}$ respectively.
This implies $\norm{\omega_1+\omega_2}_\infty = \omega_1[n]+\omega_2[n] \geq \alpha_1+\alpha_2$.

Similar arguement can be applied when $s<0$.
\end{proof}

Note that this theorem shows that the rectangular tiling mentioned in Equation (3) of the main text is actually operating under a PNF-controllable set of subbands.
We also provide an illustration of these two subband binary functions in \cref{fig:subband-mult}.
\begin{figure*}[t]
\begin{center}
\begin{tabular}{c@{}c@{}c@{}c|c@{}c@{}c@{}c}
    \includegraphics[width=0.12\linewidth]{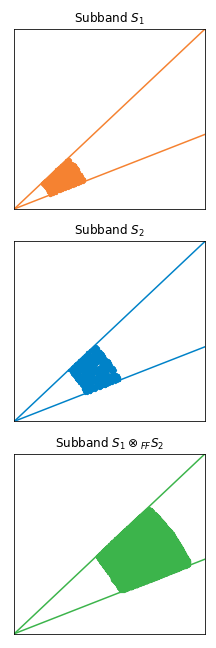}&
    \includegraphics[width=0.12\linewidth]{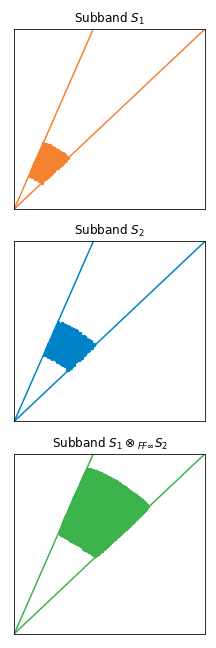}&
    \includegraphics[width=0.12\linewidth]{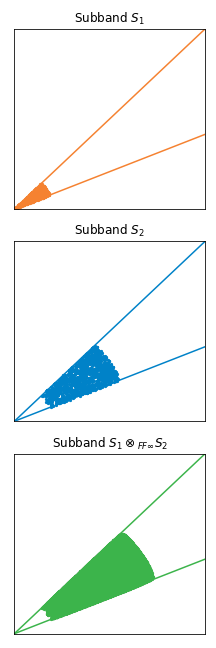}&
    \includegraphics[width=0.12\linewidth]{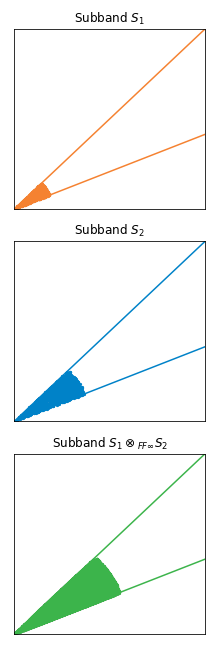}&
    \includegraphics[width=0.12\linewidth]{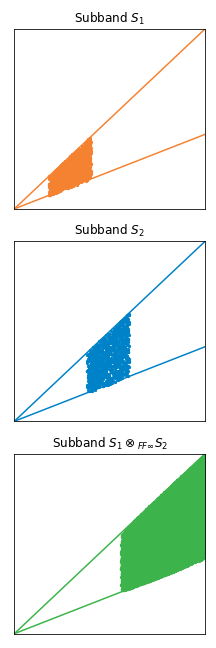}&
    \includegraphics[width=0.12\linewidth]{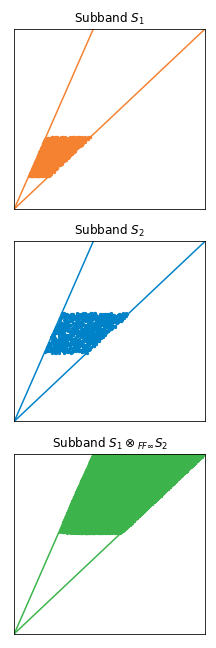}&
    \includegraphics[width=0.12\linewidth]{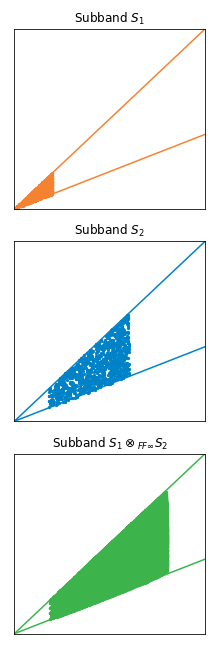}&
    \includegraphics[width=0.12\linewidth]{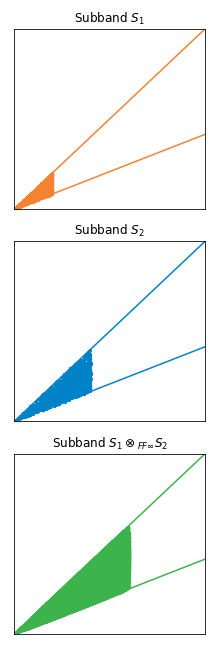}
\end{tabular}
\caption{Illustration of how subbands within PNF-Controllable Sets of subbands transform under multiplication. 
    We illustrate subbands defined in \cref{def:ff-subband}. So x-axis is $\omega[1]$ and y-axis is $\omega[2]$.
    The left four figures shows the operation under $\otimes_{FF}$ (\cref{thm:ff-control}) and the right four figures shows operations under $\otimes_{FF\infty}$ (\cref{thm:ff-control-linf}).
    The top two rows are $S_1$ and $S_2$, and the last row is $S_1 \otimes S_2$ with $\otimes$ to be the corresponding binary functions defined in the PNF-controllable set of subbands.
    }
\label{fig:subband-mult}
\end{center}

\end{figure*}

\subsubsection{RBF}

\begin{theorem}[RBF PNF-Controllable Set of Subbands]
Assume the RBF-basis functions for $\Re^d\to \Re$ in the form $\bB_{RBF} = \{\exp(-\frac{1}{2}\gamma\norm{x-\mu})\}$ parameterized by tuple $(\gamma, \mu)$, where $\gamma \in \Re$ and $\mu \in \Re^d$.
Define subband in the following way:
\begin{align}
    S_{RBF}(M,\gamma) = \{\exp(-\frac{1}{2}\gamma\norm{x-\mu}) | \mu \in \operatorname{Cvx}(M)\},
\end{align}
where $M \subset \Re^d$ and $\operatorname{Cvx}(M)$ denotes the convex hull using all vectors of $M$.
Then
\begin{align}
\sS_{RBF}=\{S_{RBF}(M,\gamma) | \forall M \subset \Re^d, |M|<\infty, \gamma > 0\}
\end{align}
is a PNF-controllabl Set of Subbands with the following definition of binary relation $\otimes_{RBF}$:
\begin{align}
    (M_1, \gamma_1) \otimes_{RBF}(M_2, \gamma_2) 
    = (M_1\cup M_2, \gamma_1 + \gamma_2).
\end{align}
\label{thm:rbf-control}
\end{theorem}
\begin{proof}
(Sketch)
This is because by applying \cref{eq:rbf-mult} , we have $\mu'$ is a weighted sum of $\mu_1$ and $\mu_2$ with weights normalized to $1$.
With that said, $\mu'$ is in the convex hull of where $\mu_1$ and $\mu_2$ is sampling from.
The operation on $\gamma$ is taken directly from \cref{eq:rbf-mult}.
\end{proof}

We can easily generalize \cref{thm:rbf-control} to subband definition where $\gamma$ samples from an interval since it only makes sense when $\gamma > 0$.
The abovementioned theory suggests that every-time we multiply two RBF basis PNF, we will increase the region of the convex hull of $M$ and increase the $\gamma$ (which is inverse to the scale of the RBF).

\subsubsection{Gabor}

We will show how to combine RBF and Fourier cases of the PNF-controllable set of subbands to create a PNF-controllable set of subband for Gabor basis.

\begin{theorem}[Multiplication Rule of PNF-Controllable Set of Subbands]
Let $\bB_1$ and $\bB_2$ be two basis for function $\Re^d\to\Re$ that satisfies \cref{def:basis}.
Let $\sS_1, \otimes_1$ and $\sS_2, \otimes_2$ be the PNF-controllable set of subbands for $\bB_1$ and $\bB_2$ respectively.
Assume the subbands for $\bB_1$ and $\bB_2$ are parameterized by $\theta_1 \in \Re^n$ and $\theta_2 \in \Re^m$ correpsondingly.
Define $\bB_3 = \{b_1b_2| b_1 \in \bB_1, b_2\in\bB_2\}$.
Define subband as $S_3(\theta_1, \theta_2) = \{b_1b_2 | b_1\in S_1(\theta_1), b_2 \in S_2(\theta_2)\}$, where $S_1$ and $S_2$ are subbands for $\bB_1$ and $\bB_2$ correspondingly.
Then following is a PNF-controllable set of subbands for basis $\bB_3$:
\begin{align}
    \sS_3 = \{S_3(\theta_1, \theta_2)| S_1(\theta_1) \in \sS_1, S_2(\theta_2) \in \sS_2\},
\end{align}
whose corresponding binary function is defined as :
\begin{align}
    (\theta_{1,a}, \theta_{2,a}) \otimes_3(\theta_{1,b}, \theta_{2,b}) = (\theta_{1,a}\otimes_1\theta_{1,b}, \theta_{2,a}\otimes_2\theta_{2,b}).
\end{align}
\label{thm:multi-control}
\end{theorem}
\begin{proof}
(Sketch)
It's easy to show that $\bB_3$ satisfies \cref{def:basis} using \cref{thm:basis-composition}.
Then we can show that multiplication of two basis function in $\bB_3$ moves $S_3$ according to $\otimes_3$ by leveraging the definition of $\sS_1, \otimes_1$ and $\sS_2, \otimes_2$.
\end{proof}

The PNF-controllable sets of subbands for Gabor basis can be obtained by applying \cref{thm:multi-control} to combine a Fourier set of subbands (e.g. \cref{thm:ff-control}) and the RBF set of subbands (e.g. \cref{thm:rbf-control}).

\subsection{Different Instantiation of PNFs}\label{sec:diff-pnfs}
In the previous sections, we've shown that 1) PNF allows different architectures, and 2) PNF allows different choices of basis, and 3) PNF can be designed to be subband-limited without losing its compositionality.
In this section, we will show several instantiation of the PNFs.
Specifically, for each design of the PNF, we will use the following steps:
\begin{enumerate}
    \item Identify the subband of interests. These subbands should be able to cover all necessary basis functions needed to reconstruct the signal correctly.
    \item Organize the subbands according to the PNF-Controllable sets of subbands.
    \item For each PNF-Controllable set of subband, create a PNF whose outputs are subband-limited to the corresponding subband of interests in the set.
    \item The final PNF is subs of all the previous PNF.
\end{enumerate}
In this section we will use the abovementioned framework to show how PNF can be instantiated in a different forms, using different basis and network architectures.
The detailed instantiation of Fourier PNF will be discussed in \cref{sec:implementationdetails}.

\paragraph{Gabor PNF.}
Similar to MFN~\cite{Fathony2021MultiplicativeFN}, we use the same network architecture as Fourier PNF, but changing the basis into a Gabor basis (as shown in \cref{lemma:gabor-basis}).
With the Gabor basis, the definition of the subbands requires to include a partition in the spatial domain as shown in \cref{thm:rbf-control}.
For simplicity, we can set $\gamma \in [0, \infty)$ (i.e not trying to control $\gamma$) and set $M$ in the following way:
\begin{align}
    M=\left\{s_i\ee_i + s_j\ee_j | i\neq j, s_{i,j}\in\{0.5,-0.5\}, \ee_i[k] = \begin{cases}1\quad\text{if } k=i \\ 0\quad\text{otherwise}\end{cases} \right\}.
\end{align}
Intuitively, $\ee_i$ is an one-hot vector with the $i^{th}$ dimension to be $1$.
An example of $M$ in 2D is $\{[0.5,0.5], [0.5,-0.5], [-0.5, 0.5], [-0.5, -0.5]\}$, which includes all the points in the rectangle of $[-0.5, 0.5]^2$.
Since all data are sampled withitn $[-0.5, 0.5]^d$, this means the only contorl we want to enforce on $\mu$ is that it should live within the boundary $[-0.5, 0.5]^d$.
With this, we are able to create a subband-limited version of Gabor MFN, whose output is a linear sums of the Gabor basis.

\paragraph{RBF PNF.}
Here we will show an interesting way to design an RBF PNF which corresponds to subdividing a rectangular grid and interpolating it with a Gaussian.
Here we will use the subband defined by \cref{def:controllable-set-of-subbands}.
We are interesting in modeling the following sets of subbands, one for each level $l$:
$S_{RBF}(l) = S_{RBF}(M_l, 2^n\gamma_0)$, where $M_l=\{\sum_{i=1}^d \frac{k_i}{2^l} \ee_i | k_i \in [-2^{l-1}, 2^{l-1})\}$ is a set of grid with resolution $2^l$.
While $M_l$ contains $2^{2l}$ number of basis, this can be created through function composition very compactly through PNF.
First, define $f(x) \in \Re^{d^2}$ that each $f(x)[i]$ is an RBF function which takes a corner in $[-0.5, 0.5]^d$ as $\mu$ and $\gamma_0$ as the scale.
Then we define the network in the following recursive way:
\begin{align}
    F_0 &= f(x),\quad F_{k+1} = (A_{k+1}F_k(x)) \odot (B_{k+1} F_k(x)),\label{eq:rbf-pnf-network}
\end{align}
where $\odot$ denotes Hardarmard product and $A, B$ are real-value matrices.
We can use \cref{thm:controllability} and \cref{thm:rbf-control} to show that $F_k$ is subband limited by $S_{RBF}(k)$ and it creates all $\mu$'s that's defined by $M_l$.
The output of $F_k$ can be viewed as an Gaussian interpolated version of points $M_l$, where the value of each $M_l$ is given by the network parameters $A$'s and $B$'s.

\paragraph{Pairwise Gabor PNF.}
For the previous definition of Gabor PNF, we are not attempting to control the spatial content of the signal (i.e. there is no ability to pull out a network output with specific $\gamma$ and $M$ region).
Fortunately, the PNF allows us to design a netowrk architecture to enable such control for both the spatial RBF part and the Fourier part.
The idea is to have a branch of the network to generate RBF PNF $H_{k,l}$ that's subband limited by different $M_k$ at different levels of $2^l\gamma_0$.
This is achievable by creating an essemble of network using equation \cref{eq:rbf-pnf-network}.
Similarly, Fourier PNF creates a series of PNFs $G_{n,m}$, each of which is subband limited in different scale (indexed by $n$) and different orientation (indexed by $m$).
We want to design the Pairwise Gabor PNF to be the sum of pairwise product of $F_{k,l}$ and $G_{n,m}$:
\begin{align}
    F(x) = \sum_{k,l,n,m} a(k,l,n,m) H_{k,l}(x)G_{n,m}(x),
\end{align}
where $a_{k,l,n,m}$ are trainable parameters.
It's easy to show that the output of such design is subband limited applying \cref{thm:multi-control}.
And this network design allows us to remove either certain spatial area by setting $a_{k,l, \cdot, \cdot} = 0$, or to remove certain frequency content by setting $a_{\cdot, \cdot, n, m}=0$.

We will show results of some of these designs in the image expressivity experiments (i.e. \cref{tab:img_fit_quant_full}.
\section{Implementation Details}
\label{sec:implementationdetails}

As mentioned in Sec.~3.3.3 of the main text,  we leverage Theorem~2 to factorize $F$
\begin{align}{\textstyle
F(\xx) = \sum_j F_j(\xx),\ &F_j(\xx) = G_{j}(\xx, b_j, b_j) W_{jn} Z_{j,n}(\xx), \label{eq:ff-pnf-within-set-def} \\
    Z_{j,1}(\xx) = G_{j}(\xx, 0, \Delta_1),\ &
    Z_{j,k}(\xx) = G_{j}(\xx, 0, \Delta_k)W_i Z_{j,k-1}(\xx),
}\end{align}
where $G_{j}(\xx, a, b)$ is subband limited in $R_F(a, b, d(\theta_j), \delta, \infty)$ and $\Delta_k = b_k - b_{k-1}$.
We instantiate this architecture by setting $G_{j}(\xx, a, b)$ into a linear transform of basis sampled from the subband to be limited:
\begin{align}
    G_{j}(\xx, a, b) = W_i\gamma_j(\xx), \gamma_j \in R_F(a, b, d(\theta_j), \delta, \infty)^d, W_i\in \Re^{h\times d},
\end{align}
where $h$ and $d$ is the dimension for the output and the feature encoding.
We realize the $W_{i}$ as linear layers with no bias. $\gamma_j \in R_F(a, b, d(\theta_j), \delta, \infty)^d$ are initialized randomly, with $\theta$  and the radius chosen uniformly in $R_F(a, b, d(\theta_j), \delta, \infty)^d$. 

For experiments involving 2D images (e.g., image fitting, texture manipulation), the output dimension $h$ for $G_{j}(\xx, b_j, b_j)$ is set to be $3$ (RGB output values). The input dimension $d$ for $Z_{j,1}(\xx) = G_{j}(\xx, 0, \Delta_1)$ is set to be $2$ ($x$ and $y$ coordinates). Otherwise the hidden dimensions are chosen to be $128$. For 3D SDF fitting, a hidden dimension of size $100$ is chosen, and for NeRF, a hidden dimension of $86$ is chosen. The input dimension for 3D SDF fitting and for NeRF is 3 (we follow BACON's setting of modeling irradiance fields). The output dimension is $1$ for 3D SDF fitting and $4$ for NeRF (RGB and Occupancy values).

\paragraph{Tiling} For images, the region of interest chosen for tiling is $[-B,B]^2$ where $B$ is band limit, set to be 64, following \bacon{}~\cite{Lindell2021BACONBC}. Eq.~2 and Eq.~3 describe a potential tiling with no overlapping fans. In practice, for fitting tasks we found it beneficial to use overlapping fans, and so consider the following, modified tilings:
\begin{align*}
\tT_{circ} = \{
S_{ij} = R_F(b_i, b'_{i+1}, \dd(\theta_j), \delta, 2) | & b_1\leq \dots \leq b_{n-1}, b'_2\leq \dots \leq b'_n,  \\
& \theta_j = j\delta,\delta=\frac{\pi}{m},1\leq j \leq 2m\},
\end{align*}
\begin{align*}
\tT_{rect} = \{
S_{ij} = R_F(b_i, b'_{i+1}, \dd(\theta_j), \delta, \infty) | & b_1\leq \dots \leq b_{n-1}, b'_2\leq \dots \leq b'_n, \\
& \theta_j = j\delta,1\leq j \leq 2m, j\neq m\}.
\end{align*}
This allows us to cover the space of frequency basis more compactly and learn a more faitfull fitting of a given signal. 
In particular, we set $b_1\leq \dots \leq b_{n-1}$ to be  
$0, \frac{1}{16}, \frac{1}{8}, \frac{1}{4}$ and $b'_2\leq \dots \leq b'_{n}$ to be  $0, \frac{1}{8}, \frac{1}{8}, \frac{1}{4}, \frac{1}{2}$. 
$m$ is chosen to be $\frac{1}{8}$, covering the frequency band with $8$ orientations. By default we use rectangular tiling for our experiments with PNF. 

For experiments that requires higher dimensional inputs (e.g. 3D), we mainly use a generalization of the rectangular tiling.
We will first create a tile for each L-$\infty$ dimensional consistent region.
In the case of 3D, we will create one PNF-controllable set of subbands for each of the following: $R_\infty(1, 1)$, $R_\infty(2, 1)$, and $R_\infty(3, 1)$.
We use the same division of bandwidth within the construction of each of these PNF-controllable subband sets, but scaled it to correpsnding max bandwidth according to different applications.
While this already covers all the basis function of interests, we found it improves the performance if we tile the frequency space in an overcomplete way.
Specifically, for each pairs of the eight octants in $[-B, B]^3$ where $B$ is the band-limit, we will create three non-overlapping PNF-controllable sets of subbands with the L-$\infty$ norm.
One interesting trick we leveraged is that the three non-overlapping PNF-controllabel sets of subband within the same octant can be implemented with one band-limited PNF.
With this said, for 3D, we will need one band-limited PNF for each axis and one for each pairs of octants.
This leads to an enssemble $7$ band-limited PNFs.

\section{Experiment Details}

\subsection{Expressivity}

\paragraph{Images}

\newcolumntype{Y}{>{\centering\arraybackslash}X}
\begin{figure*}
\begin{center}
\begin{tabular}{c@{}c@{}c@{}c@{~~~}c@{}c@{}c@{}c@{}c}
\includegraphics[width=0.12\linewidth,height=0.12\linewidth]{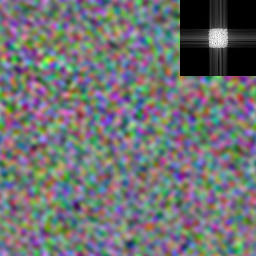} &
\includegraphics[width=0.12\linewidth,height=0.12\linewidth]{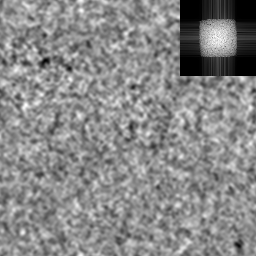} &
\includegraphics[width=0.12\linewidth,height=0.12\linewidth]{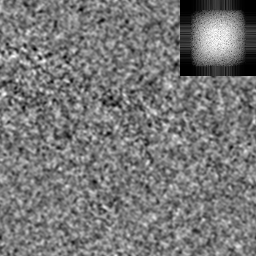} &
\includegraphics[width=0.12\linewidth,height=0.12\linewidth]{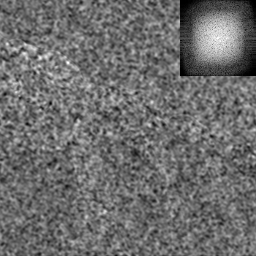} &
\includegraphics[width=0.12\linewidth,height=0.12\linewidth]{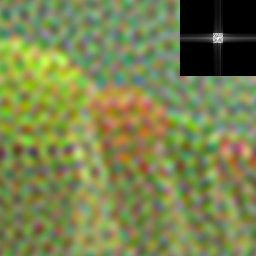} &
\includegraphics[width=0.12\linewidth,height=0.12\linewidth]{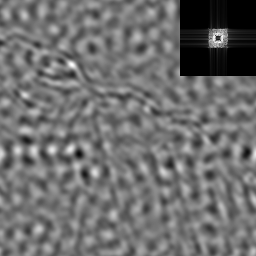} &
\includegraphics[width=0.12\linewidth,height=0.12\linewidth]{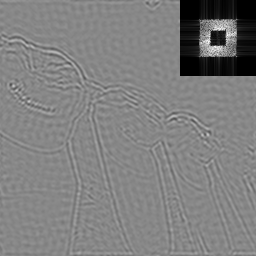} &
\includegraphics[width=0.12\linewidth,height=0.12\linewidth]{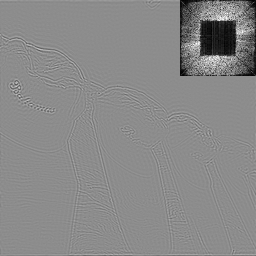} \\

\includegraphics[width=0.12\linewidth,height=0.12\linewidth]{figures/laplacian_supp1/bacon/03_BACON_mean_gau_0.png} &
\includegraphics[width=0.12\linewidth,height=0.12\linewidth]{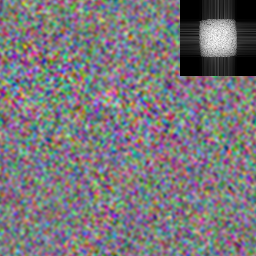} &
\includegraphics[width=0.12\linewidth,height=0.12\linewidth]{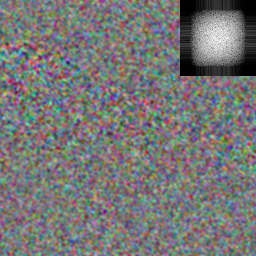} &
\includegraphics[width=0.12\linewidth,height=0.12\linewidth]{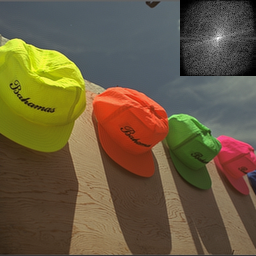} &
\includegraphics[width=0.12\linewidth,height=0.12\linewidth]{figures/laplacian_supp1/pnf/03_pnf_PNF_pseudopolar_fix2_gau_0.png} &
\includegraphics[width=0.12\linewidth,height=0.12\linewidth]{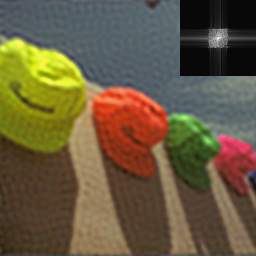} &
\includegraphics[width=0.12\linewidth,height=0.12\linewidth]{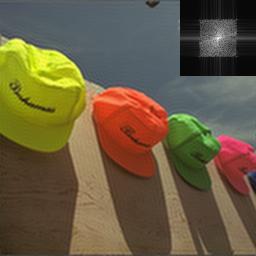} &
\includegraphics[width=0.12\linewidth,height=0.12\linewidth]{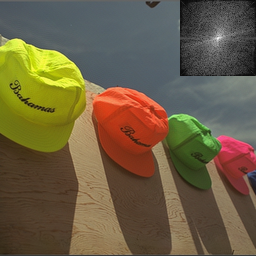}
 \\
 \multicolumn{4}{c}{BACON} & \multicolumn{4}{c}{PNF} \\ 
 
 \includegraphics[width=0.12\linewidth,height=0.12\linewidth]{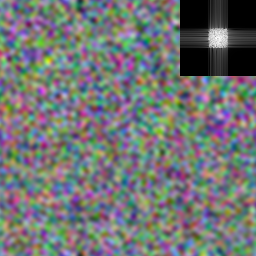} &
\includegraphics[width=0.12\linewidth,height=0.12\linewidth]{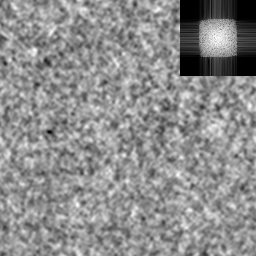} &
\includegraphics[width=0.12\linewidth,height=0.12\linewidth]{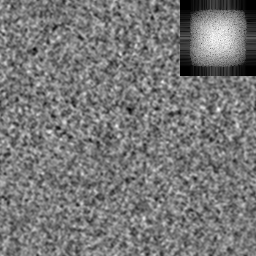} &
\includegraphics[width=0.12\linewidth,height=0.12\linewidth]{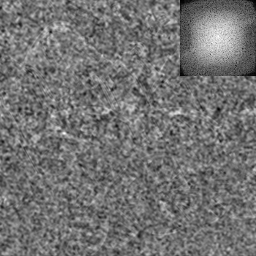} &
\includegraphics[width=0.12\linewidth,height=0.12\linewidth]{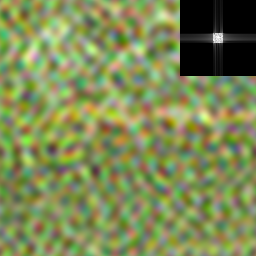} &
\includegraphics[width=0.12\linewidth,height=0.12\linewidth]{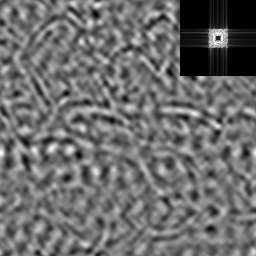} &
\includegraphics[width=0.12\linewidth,height=0.12\linewidth]{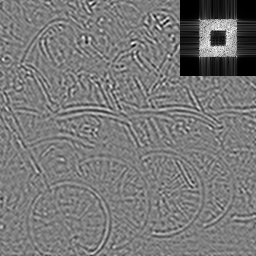} &
\includegraphics[width=0.12\linewidth,height=0.12\linewidth]{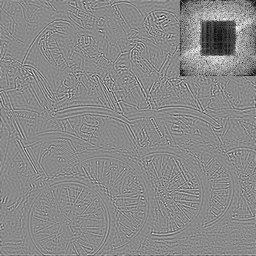} \\

\includegraphics[width=0.12\linewidth,height=0.12\linewidth]{figures/laplacian_supp2/bacon/05_BACON_mean_gau_0.png} &
\includegraphics[width=0.12\linewidth,height=0.12\linewidth]{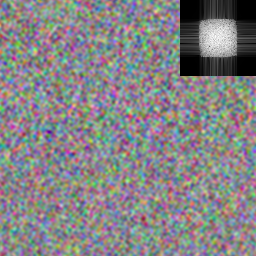} &
\includegraphics[width=0.12\linewidth,height=0.12\linewidth]{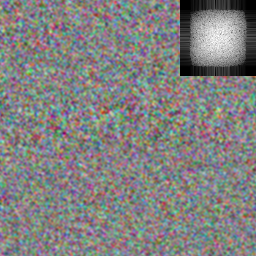} &
\includegraphics[width=0.12\linewidth,height=0.12\linewidth]{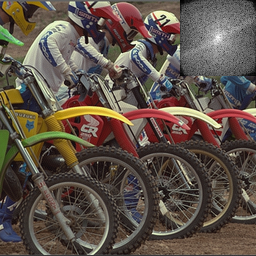} &
\includegraphics[width=0.12\linewidth,height=0.12\linewidth]{figures/laplacian_supp2/pnf/05_pnf_PNF_pseudopolar_fix2_gau_0.png} &
\includegraphics[width=0.12\linewidth,height=0.12\linewidth]{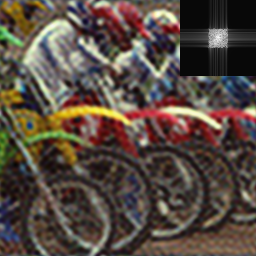} &
\includegraphics[width=0.12\linewidth,height=0.12\linewidth]{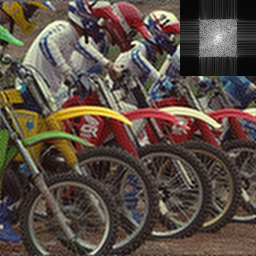} &
\includegraphics[width=0.12\linewidth,height=0.12\linewidth]{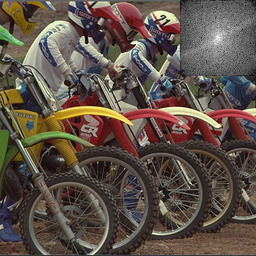}
 \\
 \multicolumn{4}{c}{BACON} & \multicolumn{4}{c}{PNF}
 
\end{tabular}
\caption{BACON: The bottom row shows the output of each layer which is upper band limited. The top row (columns 2-4) shows the difference between the output of a given layer and the one before it. PNF: The top row shows the output of each layer which is both upper and lower band limited. The bottom row (columns 2-4) shows the addition of the output at a given layer and the one before it. }
\label{fig:laplacian-pyramid_supp_1}
\end{center}
\end{figure*}
\begin{figure}
\centering
\caption{Image Fitting on the DIV2K dataset. BAC stands for BACON. Rectangular stands for PNF using $\tT_{rect}$ tiling, and Circular stands for PNF using $\tT_{circ}$ tiling. Gabor is Gabor PNF and PWGabor is the Pairwise Gabor PNF, both are described in \cref{sec:diff-pnfs}}
\begin{tabular}{@{}lll@{}}
\toprule
Method  & PSNR  & SSIM   \\ \midrule
RFF     & 28.72 $\pm$ 2.88   & 0.834  $\pm$ 0.053\\
SIREN   & 29.22 $\pm$ 3.08   & 0.866  $\pm$ 0.053\\ 	
BAC     & 28.67 $\pm$ 2.83     & 0.838 $\pm$ 0.048\\	
BAC-L   & 29.44 $\pm$ 3.03     & 0.871$\pm$ 0.047 \\ 	
BAC-M   & 29.44 $\pm$ 3.03	    & 0.871 $\pm$ 0.047 \\ 
\midrule
PNF (Rectangular)   & \textbf{29.47 $\pm$ 3.08 } & 0.874 $\pm$ 0.047 \\
PNF (Circular)      & 28.83 $\pm$ 3.04  & 0.856	$\pm$ 0.052 \\
PNF (Gabor)         & 29.39	$\pm$ 2.97  & \textbf{0.875 $\pm$ 0.047} \\
PNF (PWGabor)       & 29.22	$\pm$ 2.95  & 0.872 $\pm$ 0.047 \\
\bottomrule
\label{tab:img_fit_quant_full}
\end{tabular}
\end{figure}

For the image fitting task (Sec.~4.1, Tab.~1), we use the DIV2K~\cite{agustsson2017ntire} dataset and downsample images to $256^2$ resolution. For evaluation, we sample the fields at $512$ and compare with the original resolution images. We compare our method against state-of-the-art neural fields of \bacon{}~\cite{Lindell2021BACONBC}, Random Fourier Features~\cite{tancik2020fourier} and SIREN~\cite{sitzmann2020siren}. \cref{tab:img_fit_quant_full} corresponds to Tab.1 of the main paper. We add here SD values corresponding to an averaged over all $25$ images in  DIV2K~\cite{agustsson2017ntire} dataset. In the paper we report values rectangular tiling. In \cref{tab:img_fit_quant_full} we also add the values corresponding to circular tiling. 
Additional results corresponding to Fig.~2 of the main paper are given in \cref{fig:laplacian-pyramid_supp_1}

\paragraph{Neural Radiance Field}
\begin{table}[t]
\centering
\caption{NeRF Fitting for $64^2$ resolution. Full results on 1x ($64^2$), 1/2, 1/4, 1/8 resolutions. 
}
\begin{tabular}{@{}lcccc@{~~~}cccc@{}}
\toprule
& \multicolumn{4}{c}{\bacon{}~\cite{Lindell2021BACONBC}} & \multicolumn{4}{c}{PNF (Ours)} \\
\cmidrule(l){2-5}\cmidrule(l){6-9}
                           & \multicolumn{2}{c}{300 epochs}                      & \multicolumn{2}{c}{500 epochs} & \multicolumn{2}{c}{300 epochs}                      & \multicolumn{2}{c}{500 epochs}                      \\ \cmidrule(l){2-3}\cmidrule(l){4-5}  \cmidrule(l){6-7}\cmidrule(l){8-9}
                           & \multicolumn{1}{l}{PSNR} & \multicolumn{1}{l}{SSIM} & \multicolumn{1}{l}{PSNR} & \multicolumn{1}{l}{SSIM} 
                           & \multicolumn{1}{l}{PSNR} & \multicolumn{1}{l}{SSIM} & \multicolumn{1}{l}{PSNR} & \multicolumn{1}{l}{SSIM}  \\ \midrule

LEGO 1x        & 29.62 & 0.969 & 30.07 & 0.973 & 30.40 & 0.967                         & 31.10 & 0.965 \\
LEGO 1/2x      & 29.90 & 0.955 & 30.65 & 0.957 & 30.41 & 0.962                         & 31.17 & 0.951 \\
LEGO 1/4x      & 29.91 & 0.939 & 30.46 & 0.939 & 31.08 & 0.947                         & 31.75 & 0.944 \\
LEGO 1/8x      & 28.87 & 0.913 & 28.85 & 0.917 & 30.39 & 0.921                         & 30.86 & 0.934 \\
\midrule
Chair 1x       & 29.89 & 0.914 & 30.93 & 0.958 & 31.05 & 0.966                         & 30.91 & 0.983 \\
Chair 1/2x     & 30.97 & 0.898 & 34.93 & 0.968 & 35.00 & 0.949                         & 34.88 & 0.975 \\
Chair 1/4x     & 30.40 & 0.943 & 37.04 & 0.973 & 37.16 & 0.937                         & 37.06 & 0.956 \\
Chair 1/8x     & 29.65 & 0.947 & 35.92 & 0.964 & 36.29 & 0.929                         & 36.55 & 0.956 \\
\midrule
Drums 1x       & 26.10 & 0.937 & 28.24 & 0.923 & 28.30 & 0.932                         & 28.12 & 0.958 \\
Drums 1/2x     & 28.38 & 0.942 & 30.20 & 0.956 & 30.05 & 0.939                         & 30.16 & 0.960 \\
Drums 1/4x     & 28.63 & 0.924 & 31.09 & 0.946 & 31.52 & 0.925                         & 31.44 & 0.955 \\
Drums 1/8x     & 32.79 & 0.950 & 32.25 & 0.940 & 32.64 & 0.954                         & 32.96 & 0.946 \\
\midrule
Ficus 1x       & 25.61 & 0.925 & 28.18 & 0.953 & 29.31 & 0.965                         & 30.34 & 0.976 \\
Ficus 1/2x     & 31.49 & 0.955 & 33.12 & 0.975 & 30.37 & 0.957                         & 30.45 & 0.962 \\
Ficus 1/4x     & 35.48 & 0.973 & 35.55 & 0.979 & 30.96 & 0.945                         & 30.33 & 0.943 \\
Ficus 1/8x     & 37.98 & 0.971 & 38.00 & 0.977 & 30.35 & 0.922                         & 28.09 & 0.912 \\
\midrule
Hotdog 1x      & 30.82 & 0.981 & 33.82 & 0.981 & 34.62 & 0.966                         & 34.80 & 0.983 \\
Hotdog 1/2x    & 29.39 & 0.973 & 34.37 & 0.974 & 34.15 & 0.949                         & 34.35 & 0.975 \\
Hotdog 1/4x    & 29.79 & 0.929 & 32.44 & 0.953 & 32.49 & 0.937                         & 32.55 & 0.956 \\
Hotdog 1/8x    & 28.72 & 0.939 & 32.28 & 0.940 & 31.87 & 0.929                         & 32.15 & 0.956 \\
\midrule
Materials 1x   & 23.67 & 0.901 & 22.17 & 0.902 & 25.13 & 0.924                         & 24.92 & 0.956 \\
Materials 1/2x & 25.70 & 0.946 & 24.17 & 0.915 & 28.23 & 0.935                         & 27.85 & 0.972 \\
Materials 1/4x & 26.70 & 0.928 & 25.46 & 0.924 & 27.01 & 0.913                         & 26.72 & 0.951 \\
Materials 1/8x & 26.01 & 0.806 & 24.47 & 0.842 & 27.68 & 0.860                         & 26.62 & 0.863 \\
\midrule
Mic 1x         & 30.53 & 0.981 & 29.67 & 0.979 & 30.55 & 0.983                         & 30.12 & 0.979 \\
Mic 1/2x       & 34.02 & 0.972 & 31.72 & 0.973 & 33.98 & 0.975                         & 33.52 & 0.976 \\
Mic 1/4x       & 34.69 & 0.965 & 33.37 & 0.954 & 35.95 & 0.957                         & 34.60 & 0.957 \\
Mic 1/8x       & 35.54 & 0.958 & 36.00 & 0.954 & 35.35 & 0.953                         & 35.52 & 0.952 \\
\midrule
Ship 1x        & 28.30 & 0.878 & 24.99 & 0.789 & 29.74 & 0.795                         & 24.35 & 0.795 \\
Ship 1/2x      & 29.74 & 0.901 & 27.02 & 0.812 & 27.11 & 0.899                         & 27.17 & 0.829 \\
Ship 1/4x      & 30.40 & 0.910 & 28.54 & 0.854 & 30.70 & 0.914                         & 27.41 & 0.828 \\
Ship 1/8x      & 31.42 & 0.931 & 29.00 & 0.892 & 31.40 & 0.891                         & 28.17 & 0.888

               \\ \bottomrule
\label{tab:nerf-expressivity-additional}
\end{tabular}
\end{table}

The method devised by NeRF~\cite{mildenhall2020nerf} can be used for novel view synthesis. It operates on a dataset of images from different views with known camera parameters. NeRF queries a neural field $M$ for an RGB and occupancy value, given a 3D point on a ray which passes through an image pixel that extends from a camera center. The RGB and occupancy values are then aggregated using standard volumetric rendering pipeline. After training, novel views are rendering by evaluating the relevant rays. For test views, evaluation is performed by measuring the difference between generated views and ground truth views, using SSIM and PSNR measures. 

For the choice of neural field $M$, we evaluate our method on the multiscale Blender dataset~\cite{barron2021mip} with images at full (64$\times$64), $1/2$, $1/4$, and $1/8$ resolution. We compare our method to \bacon{} which is state-of-the-art on this task. We use the same training scheme as in \bacon{} for this task. In Fig.~4, we also provide a visual comparison for the drums scene, trained at full (512$\times$512), $1/2$, $1/4$, and $1/8$ resolution. 

An Adam optimizer is used for training with $1e6$ training iterations. Learning rate is annealed logarithmically from $1e-3$ to $5e-6$. For BACON, $8$ hidden layers are used with $256$ hidden features. As mention in \cref{sec:implementationdetails}, for our network, $4$ layers are used with hidden dimension of $86$. This results in the total memory which is slightly below that of \bacon{}. 
Rays for the multiscale Blender dataset are in $[-4, 4]^3$. We follow follow \bacon{} in setting the maximum bandwidth to be $64$ cycles per unit interval and in evaluating without the viewing direction as input. We also adapt the hierarchical sampling of NeRF~\cite{mildenhall2020nerf}. For a fair comparison to \bacon{} we consider \bacon{}'s per-scale supervision using the loss of $\sum\limits_{i, j, k} \lVert (\mathbf{I}_k(\mathbf{r}_i, \mathbf{t}_j) - \mathbf{I}_{\text{GT},k}(\mathbf{r}_i) \rVert^2_2$, for $i$, $j$, and $k$ being index rays, ray positions, and dataset scales. 
Tab.~3 of the main text provides the result averaged over all scenes in the Blender dataset for $1x$ (64$\times$64) resolution and for the average over $1x$, $1/2$, $1/4$, and $1/8$ resolution. Full results are provided in \cref{tab:nerf-expressivity-additional}.

\paragraph{3D Signed Distance Field}
\begin{table}[]
\centering
\caption{3D shape fitting. CD is Chamfer Distance ($\times 10^{6}$). FS stands for F-score and NC stands for normal consistency.}
\begin{tabular}{llccc@{~~~}ccc}
\toprule
                       &                        & \multicolumn{3}{c}{Armadillo}                                                                        & \multicolumn{3}{c}{Dragon}                                                                           \\
                       
                       \cmidrule(l){3-5} \cmidrule(l){6-8}
                       &         \#Iters               & \multicolumn{1}{c}{CD}           & \multicolumn{1}{c}{NC}          & \multicolumn{1}{c}{FS}          & \multicolumn{1}{c}{CD}           & \multicolumn{1}{c}{NC}          & \multicolumn{1}{c}{FS}          \\
                       \midrule
SIREN                  &                        & 2.86                         & 99.16                         & 99.92\%                         & 11.9 & 97.03 & 99.43\% \\
\midrule
BACON                  &                        & 2.86                         & 99.12                         & 99.91\%                         & 9.57                         & 96.94                         & 99.65\%                         \\
BACON-last             &                        & 2.88                         & 99.05                         & 99.91\%                         & 5.11                         & 96.92                         & 99.67\%                         \\
BACON-mean             & \multirow{-4}{*}{200k} & 2.86                         & 99.07                         & 99.91\%                         & 9.80                         & 96.83                         & 99.57\%                         \\
\midrule
                       & 200k                   & 2.87                         & 99.14                         & 99.91\%                         & 1.72                         & 97.02                         & 99.99\%                         \\
                       & 100k                   & 2.86                         & 99.14                         & 99.92\%                         & 1.73                         & 96.99                         & 99.99\%                         \\
                       & 60k                    & 2.87                         & 99.13                         & 99.91\%                         & 1.74                         & 96.93                         & 99.99\%                         \\
                       & 40k                    & 2.88                         & 99.11                         & 99.90\%                         & 1.76                         & 96.84                         & 99.99\%                         \\
\multirow{-5}{*}{Ours} & 20k                    & 2.91                         & 99.05                         & 99.91\%                         & 1.82                         & 96.67                         & 99.99\%                         \\
\midrule
Oracle                 &                        & 2.85                         & 99.15                         & 99.92\%                         & 1.67                         & 97.84                         & 100.00\%                        \\ \bottomrule \\
                       &                        & \multicolumn{3}{c}{Lucy}                                                                             & \multicolumn{3}{c}{Thai}                                                                             \\
                       \cmidrule(l){3-5} \cmidrule(l){6-8}
                       &       \#Iters                 & \multicolumn{1}{c}{CD}           & \multicolumn{1}{c}{NC}          & \multicolumn{1}{c}{FS}          & \multicolumn{1}{c}{CD}           & \multicolumn{1}{c}{NC}          & \multicolumn{1}{c}{FS}          \\
                       \midrule
SIREN                  &                        & 18.7 & 98.00 & 99.74\% & 2.61                         & 94.71                         & 99.96\%                         \\
\midrule
BACON                  &                        & 18.1                         & 97.79                         & 99.82\%                         & 2.60                         & 94.57                         & 99.97\%                         \\
BACON-last             &                        & 4.95                         & 97.03                         & 99.73\%                         & 2.60                         & 94.59                         & 99.97\%                         \\
BACON-mean             & \multirow{-4}{*}{200k} & 22.1                         & 97.94                         & 99.94\%                         & 2.61                         & 94.42                         & 99.97\%                         \\
\midrule
                       & 200k                   & 1.75                         & 97.89                         & 100.00\%                        & 2.61                         & 94.56                         & 99.97\%                         \\
                       & 100k                   & 1.75                         & 97.87                         & 100.00\%                        & 2.60                         & 94.57 & 99.96\%                         \\
                       & 60k                    & 1.76                         & 97.81                         & 100.00\%                        & 2.64                         & 94.45                         & 99.96\%                         \\
                       & 40k                    & 1.77                         & 97.78                         & 100.00\%                        & 2.67                         & 94.30                         & 99.97\%                         \\
\multirow{-5}{*}{Ours} & 20k                    & 1.81                         & 97.66                         & 100.00\%                        & 2.75                         & 93.95                         & 99.96\%                         \\
\midrule
Oracle                 &                        & 1.71 & 98.54 & 99.99\% & 2.54                         & 95.55                         & 99.97\%               \\          
\bottomrule
\end{tabular}
\label{fig:sdf-all-shape-numbers}
\end{table}
As mentioned in the main text, we evaluate the performance of our method against the Stanford 3D scanning repository\footnote{\url{http://graphics.stanford.edu/data/3Dscanrep/}}. In Tab.~2 of the main text we report the averaged for the scenes of \textit{Armadillo}, \textit{Dragon}, \textit{Lucy}, and \textit{Thai Statue}. Individual per-object scores are provided in \cref{fig:sdf-all-shape-numbers} where we consider an additional evaluation metric of normal consistency. Normal consistency (NC), first computes the nearest points using Campfer Distance and then then computes whether the surface normal is within certain threshold. NC is the percentage that lands within the threshold. In \cref{fig:sdf-all-shape-numbers}, oracle corresponds to the upper bound for the performance, computed by sampling two sets of points form the ground truth and computing the evaluation metrics on them. 
We compare our method to SIREN~\cite{sitzmann2020siren} and \bacon{}~\cite{Lindell2021BACONBC}. We train each network to fit a signed distance function (SDF). For \bacon{} and SIREN, $8$ hidden layers are used with $256$ hidden features.  The models are extracted at $512^3$ resolution using marching cubes and evaluated using F-score and Chamfer distance. 

We follow a similar training procedure to \bacon{}~\cite{Lindell2021BACONBC}. Training data consists of sampled locations from the zero level set, where Laplacian noise is added for each point, as in \cite{davies2020effectiveness}. As noted in \bacon{}, the width of the Laplacian has a large performance impact, and so we use the same coarse and fine sampling procedure of \bacon{}, where ``fine'' samples are produced using a small variance of  $\sigma^2_\text{L}=\text{2e-6}$ and ``coarse'' samples with $\sigma_\text{L}^2=\text{2e-2}$. Samples are drawn in the domain $[-0.5, 0.5]^3$ and the following loss is used: $\lambda_\text{SDF} \lVert \mathbf{y}^c - \mathbf{y}^c_\text{GT}\rVert_2^2 + \lVert \mathbf{y}^f - \mathbf{y}^f_\text{GT} \rVert_2^2$, where 
$\mathbf{y}$ is the generated output, $\mathbf{y}_\text{GT}$ is the ground truth value, the $f$ and $c$ indicate fine and coarse samples. As in \bacon{}, $\lambda_\text{SDF}$ is set to 0.01 for all experiments.  SIREN and \bacon{} are used as baselines, each trained for $200,000$ with a batch size of $5000$ coarse and $5000$ fine samples. The same optimization as for neural radiance field is used here.

\subsection{Texture Transfer}

\begin{figure}
\begin{center}
\begin{tabular}{c@{~~}lc@{~~}c@{}c@{}cc@{~~}c@{}c@{}c@{}}

& & & L-1 & L-2 & L-3 & & L-1 & L-2 & L-3 \\  
\rotatebox[origin=l]{90}{~~Texture} & 
\includegraphics[width=0.118\linewidth]{figures/texture/wooden_128.jpg} &
\rotatebox[origin=l]{90}{~~BACON} & 
\includegraphics[width=0.118\linewidth]{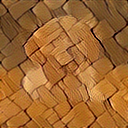} &
\includegraphics[width=0.118\linewidth]{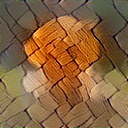} &
\includegraphics[width=0.118\linewidth]{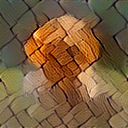} & 
 
\rotatebox[origin=l]{90}{~~~~PNF} & 
\includegraphics[width=0.118\linewidth]{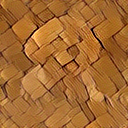} &
\includegraphics[width=0.118\linewidth]{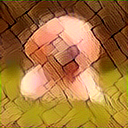} &
\includegraphics[width=0.118\linewidth]{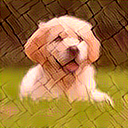} 
\\
\end{tabular}
\vspace{-0.2cm}
\caption{Additional texture transfer results.  We optimize specific layers of the neural field. L-1 (layers 1-4), L-2 (layers 2-4), L-3 (layers 3-4). }
\label{fig:texture_additional}
\end{center}
\end{figure}

For texture transfer, a random cat image is used (free-use licence) and texture image is used from \cite{efros2001image}. As noted in the main text, we query the network on a $128^2$ grid and apply either (a) Content and style loss objectives as given in \cite{gatys2015neural}, or (b) Text-based texture manipulation objectives as given by CLIPStyler~\cite{kwon2021clipstyler}. For baseline comparison, the same objectives are applied directly on the cat image, and the cat image pixel values are directly optimized. 

For (a), we use the same objectives (content and style loss), loss weights and optimizer as used in the public code of \url{https://pytorch.org/tutorials/advanced/neural_style_tutorial.html} and \cite{gatys2015neural}. Similarly, for (b), we similarly use the same we objectives, loss weights and optimizer as used in the official code of CLIPStyler~\cite{kwon2021clipstyler} in \url{https://github.com/cyclomon/CLIPstyler}. 
Additional texture transfer corresponding on the Fig.~5(a) is given in \cref{fig:texture_additional}.

\subsection{Scale-space Representation}
In this section, we will first show the derivation of the PNF used for the scale-space representation.
Then we will provide detailed description of the experiment and provide numerical results.

\subsubsection{Derivation of Scale-space Fourier PNF}
Suppose the signal of interest can be represented by Fouier bases as $g(\xx) = \sum_n \alpha_n \exp{(\omega_i^T\xx)}$, then we know analytically the Gaussian convolved version should be $f(\xx, \Sigma) = \sum_n \alpha_n\exp(\omega_i^T\Sigma\omega_i)\exp(i\omega_i^T\xx)$.
If we assume that our Fourier PNF can learn the ground truth representation well, then one potential way to achieve this is setting $\gamma(\xx, \Sigma)_n=\exp(-\frac{1}{2}\omega_n^T\Sigma\omega_n)\exp(-i\omega_n^T\xx)$.
Doing this naively with Fourier PNF only yields a bad approximation, since each coefficient is off by an error term $E(n, I_n)$ from the ground truth.
We will derive such error term in detail below.

For a particular $F_k$ (Equation (4) in the main text), let $\tilde{F}_k(\xx, \Sigma)$ be the output of naively replacing the Fourier basis encoding $\gamma$ with the intergrated Fourier basis encoding: $\gamma(\xx, \Sigma)$.
Then we have:
\begin{align}
    &\ \ \mathop{\mathbb{E}}_{\xx\sim\mathcal{N}(\mu,\Sigma)}[F_k(\xx)]  \\
    &= 
    \sum_n \alpha_n
    \exp\left(-\frac{1}{2}\omega_n^T\Sigma\omega_n^T \right) \exp(i\omega_n^Tx)  \\
    &=\sum_n\exp\left(-\frac{1}{2}\omega_n^T\Sigma\omega_n^T \right)
    \sum_{I \in I_n} \alpha_{nI} \exp\paren{i\paren{\sum_{j\in I} \omega_j}^T\xx}
    \\
    &=\sum_n
    \sum_{I \in I_n} 
    \exp\left(-\frac{1}{2}\omega_n^T\Sigma\omega_n^T \right)
    \alpha_{nI} \exp\paren{i\paren{\sum_{j\in I} \omega_j}^T\xx}
    \\
    &=\sum_{n,I \in I_n} \alpha_{nI}
    \exp\paren{-\frac{1}{2}\sum_{j \in I} \omega_j^T\Sigma\omega_j^T}
    \exp\paren{-\frac{1}{2}\sum_{j\neq l, j,l \in I} \omega_j^T\Sigma\omega_l^T}
     \exp\paren{i\paren{\sum_{j\in I} \omega_j}^T\xx}
    \\
    &= \underbrace{
        \sum_{n,I \in I_n} \alpha_{nI} 
        \paren{
        \prod_{l\in I}
        \exp\paren{-\frac{1}{2}\omega_l^T\Sigma\omega_l^T + i\omega_n^T\xx}
        }
    }_{\text{PNF output:} \tilde{F}_k(\xx, \Sigma)}
    \underbrace{\exp\left(-\frac{1}{2}\sum_{l,j \in I, l\neq j} \omega_l^T\Sigma\omega_j^T \right)}_{\text{Error terms } E_k(n, I_n)}
    ,
\end{align}
where $I_n$ include all indexes to choose one basis per $\gamma_j$ layer (as defined in Equation (6) of the main text) such that $\forall I \in I_n, \sum_{l\in I} \omega_l = \omega_n$.
We use $\alpha_{nI}$ to denotes the coefficients gathered along the index of $I$.
We can compute these coefficients analytically by applying similar analysis as MFN~\cite{Fathony2021MultiplicativeFN}.
The formua above shows that $\tilde{F}_k(\xx, \Sigma)$ is not a good approximation since each term $\alpha_{n,I}$ is off by different factor $E_k(n, I_n)$.

Fortunately, we know that all bases chosen from $I_n$ must come from a specific set of that shares the same angle (i.e. the limiting-subbands all shared a direction vector $\dd$ and an angular width $\gamma$ by the definition of $F_i$), but with different range of for the norm $\norm{\omega}_p$.
This suggests that we can approximate the error terms in the following way:
\begin{align}
    E_i(\cdot, I) \approx
    \exp\left(-\frac{1}{2}\dd_{I}^T\Sigma\dd_{I}^T\sum_{k,l \in I, k\neq l}\bar{r}_k\bar{r}_l \right) = A_i(\Sigma),
\end{align}
where $\bar{r}_k, \bar{r}_l$ are the medium radius of each subband within the subband series $\sS_i$.
Specifically, $\bar{r}_l = \frac{1}{2}(\alpha_l + \beta_l)$ if the $l^{th}$ subband used for $F_k$ is $S_l\in\sS_i$ lower- and upper- bounded by $\alpha_l$ and $\beta_l$.
This way, the error estimation only depends on the PNF-controllable set of subbands used to developed the network $F_k$, but not the network parameter.
This means we can modify the network architecture of each $F_i$ in test-time to compute $A_k(\Sigma)$ to approximate error term $E_k(\cdot, I)$:
\begin{align}
&\ \ \mathop{\mathbb{E}}_{\xx\sim\mathcal{N}(\mu,\Sigma)}[F_k(\xx)]  \\
&\approx \sum_{n,I\in I_n} \alpha_{nI} 
        \paren{
        \prod_{l\in I}
        \exp\paren{-\frac{1}{2}\omega_l^T\Sigma\omega_l^T + i\omega_n^T\xx}}
    A_k(\Sigma) \\
    &= A_i(\Sigma) \sum_{n,I\in I_n} \alpha_{nI} 
    \sum_{n,I \in I_n}
        \paren{
        \prod_{l\in I}
        \exp\paren{-\frac{1}{2}\omega_l^T\Sigma\omega_l^T + i\omega_n^T\xx}} \\
    &= A_i(\Sigma) \tilde{F}_k(\xx, \Sigma)
\end{align}
With this, our scale-space interpolation Fourier PNF can be written as sums of outputs of all subband series $\sS_k$:
\begin{align}
    \tilde{F}(\xx, \Sigma)=\sum_{k}A_k(\Sigma)\tilde{F}_k(\xx, \Sigma).
\end{align}
This approximation will be exact if $\gamma=1$ and $\alpha=\beta$ for all subbands.
But it also means we need infinite number of $F_k$ to tile the space.
The smaller $\gamma$ is, the more error can this potentially occurs.
Similarly, the larger $|\alpha-\beta|$ is, the more error this approximate can creates.
This shows a trade-off between compute and the approximation error.

\begin{table}[t]
\centering
\caption{Numerical results for scale-space interpretation experiments. IPE: intergrated positional encoding~\cite{barron2021mip}; RIPE: intergrated positional encoding with randomly fourier features; IPE-sup: supervised IPE with both 1x and 1/4x resolution. While PNF is only supervised with 1x, it's capable of inerpolating into smaller scale without breaking the image structure (as shown in good SSIM). This is more impressive as the basis functions for PNF is sampled from random directions (i.e. not necessarily aligned with the eigenvectors directions of the test-time Gaussian covariances). Doing the same thing to IPE results in worse performance as shown in RIPE.}
\label{tab:scale-space-2d}
\begin{tabular}{@{}lccccccccc@{}}
\toprule
        & \multicolumn{4}{c}{PSNR}      && \multicolumn{4}{c}{SSIM}      \\ 
        \cmidrule{2-5}\cmidrule{7-10}
Model   & 1x    & 1/2x  & 1/4x  & 1/8x  && 1x    & 1/2x  & 1/4x  & 1/8x  \\\midrule
RIPE    & 39.41 & 6.35  & 6.54  & 6.63  && 96.95 & 20.43 & 22.73 & 23.8  \\
IPE     & 37.58 & 24.62 & 21.94 & 14.68 && 94.67 & 73.88 & 58.51 & 37.95 \\
BACON   & 40.63 & 8.88  & 7.19  & 7.2   && 97.39 & 43.11 & 28.16 & 29.66 \\
PNF     & 36.45 & \textbf{27.07} & 29.74 & \textbf{24.6} && 95.98 & \textbf{86.30} & 83.33 & \textbf{72.66} \\ \midrule
RIPE-sup & \textbf{50.11} & 11.91 & 26.64 & 10.69 && \textbf{99.62} & 15.51 & 73.24 & 11.27  \\
IPE-sup & 34.76 & 26.8  & \textbf{30.00}    & 18.41 && 90.97 & 84.81 & \textbf{89.65} & 65.7  \\
\bottomrule
\end{tabular}
\end{table}
\subsubsection{Details}
The image used for Scale-space Representation is taken from the Set12 dataset~\cite{zhang2017beyond} and downsampled to $128^2$ resolution. 
We train all models for 1500 iterations with early stopping when the training PSNR reaches 40.
All models use Adam optimizer with learning rate $1e-3$.
The ground truth is generated from applying Gaussian filter of size $3\times3$ (for 1/2x), $5\times5$ (for 1/4x), and $65\times65$ (for 1/8x).
We compute the ground truth using OpenCV~\cite{opencv_library}.
For $\Sigma$, we use anisotropic covariance matrix of the form $\Sigma = \sigma^2 I$ with $I$ being the $2\times 2$ identity matrix.
We set $\sigma^2 = 1$ for 1x, $\sigma=2$ for 1/2x, $\sigma=4$ for 1/4x, and $\sigma=8$ for 1/8x.
For PNF, we used the spherical tiling $\tT_{rect}$ with $10$ subbands with 2-times overcomplete. 
So each subband has angular width of $\frac{2}{12}$ as we need to restrict the subband into either total-vertical or total-horizontal area (i.e. $R_\infty(1,+1)$ or $R_\infty(0, +1)$ as definted in \cref{def:l-inf-consistent}).
The band limits for each intermediate residual output are set to be the following fraction of the Nyquist rates: $[0, \frac{1}{8}]$, $[\frac{1}{16}, \frac{1}{4}]$, $[\frac{1}{8},\frac{1}{2}]$, and $[\frac{1}{4}, 1]$.
In this experiment, we compare to the following baselines:
\begin{enumerate}
    \item RIPE: we use intergrated random Fourier features in the following form: $\exp\paren{-\frac{1}{2}\omega^T\Sigma\omega}\sin(\omega^T\xx)$. These intergrated random Fourier features will be followed by a ReLUMLP.
    \item IPE: we use ReLUMLP on axis-aligned intergrated positional encodings as shown in \cite{barron2021mip}. 
    \item BACON: replacing BACON's filter with the intergrated random Fourier Features.
    \item RIPE-sup: same as RIPE, but supervised with both 1x and 1/4x.
    \item IPE-sup: same as IPE, but supervised with both 1x and 1/4x.
\end{enumerate}
All baselines have about 0.3M parameters (which is about the same number of floating points as the 128x128 image).
All models are trained on 1x images except for IPE-sup, which is also supervised for 1/4x.
The results are shown in \cref{tab:scale-space-2d}.
As we can see from the table, PNF is capable of representing 1/2x, 1/4x, and 1/8x resonably well while supervised only on 1x resolution.
Comparing to IPE-sup, PNF is able to match the performance at 1/4, eventhough PNF is not supervised at 1/4x.
Moreover, PNF used random Fourier features instead of the axis-align basis functions.
As shown in the comparison between RIPEs and IPEs, axis-align positional encoding performs better in this setting, which we hypothesize the reason being that during the test-time the $\Sigma$'s eigenvectors are also axis-align.
PNF model choose direction randomly within the angle of the PNF-Controllable set of subbands, and it still achieves very strong performance.

\subsection{Licenses}
DIV2K~\cite{agustsson2017ntire} dataset is available for academic use only. The multiscale Blender dataset~\cite{barron2021mip} is provided under the creative commons public license. The Stanford 3D scanning repository\footnote{\url{http://graphics.stanford.edu/data/3Dscanrep/}} dataset is available to use for research purposes. Data used for texture transfer is available for academic use or is free to use. Set12 dataset~\cite{zhang2017beyond} is available under under the creative commons public license. All data does not contain personally identifiable
information or offensive content. 

\subsection{GPU Resources}

Image based experiments require a single GPU such as an NVIDIA Titanx/Titanrtx GPU. For Neural fields and 3D Signed Distance Fields, a higher GPU number is required per experiment, and up to 4 NVIDIA Titanrtx GPUs. An internal GPU cluster was used for experiments. 

\subsection{Negative Societal Impact}

On some experimental settings, our framework requires a large GPU and memory requirement, which may have an adverse  environmental impact. Neural fields have been used in a verity of generative settings, our framework may be used in potential negative use of such generative models. For instance, for neural radiance fields, creating realistic novel views of fake scenes, could be used maliciously. Such cases could be better handled by developing tools to detect fake scenes, or any other content generated using neural fields such as our PNF.

\end{document}